\documentstyle{article}

\newtheorem{Proposition}{Proposition}

\newtheorem{Theorem}{Theorem}
\newtheorem{Corollary}{Corollary}

\newenvironment{proof}
   {\begin{trivlist} \item[] {\bf Proof.}}{\hfill \rule{.3em}{2ex} \end{trivlist}}

\newcommand{\memberof}{\,{\in}\,}
\newcommand{\define}{\stackrel{{\rm df}}{=}}

\newcommand{\tensorimplysource}
   {\! \mbox{\hspace{.21em}--\hspace{-.21em}} \backslash }
\newcommand{\tensorimplytarget}
   { / \mbox{\hspace{-.23em}--\hspace{.23em}} \!}

\newcommand{\powerof}[1]{ {\bf 2}^{#1} }
\newcommand{\subpowerof}[1]{ {\rm 2}^{#1} }

\newcommand{\closure}[2]{{#1}^{\Leftrightarrow}_{#2}}

\newcommand{\product}[2]{ {#1} {\times} {#2} }

\newcommand{\pair}[2]{\langle #1,#2 \rangle}
\newcommand{\relcopair}[4]{{[{#1},{#2}]}^{#3}_{#4}}
\newcommand{\relpair}[4]{{\left({#1},{#2}\right)}^{#3}_{#4}}

\newcommand{\triple}[3]{\mbox{$ \langle #1,#2,#3 \rangle $}}
\newcommand{\quadruple}[4]{\mbox{$ \langle #1,#2,#3,#4 \rangle $}}

\newcommand{\term}[3]{{#1} \stackrel{{#2}}{\rightharpoondown} {#3}}

\newcommand{\upper}[2]{ {#1}^{{\rm u}}_{#2} }
\newcommand{\looer}[2]{ {#1}^{{\rm l}}_{#2} }
\newcommand{\derivedir}[2]{ {#1}^{\Rightarrow}_{#2} }
\newcommand{\deriveinv}[2]{ {#1}^{\Leftarrow}_{#2} }
\newcommand{\idealexists}[2]{ {\langle {#1} \rangle}^{\rm co} {#2} }
\newcommand{\filterexists}[2]{ {\langle {#1} \rangle} {#2} }
\newcommand{\idealforall}[2]{ {[ {#1} ]}^{\rm co} {#2} }
\newcommand{\filterforall}[2]{ {[ {#1} ]} {#2} }
\newcommand{\morphism}[3]{{#1} \stackrel{{#2}}{\rightarrow} {#3}}
\newcommand{\longmorphism}[3]{{#1} \stackrel{{#2}}{\longrightarrow} {#3}}
\newcommand{\Morphism}[3]{{#1} \stackrel{{#2}}{\Rightarrow} {#3}}

\newcommand{\vdashv}{ \,{\vdash}\!\!{\dashv}\, }

\title{Conceptual Collectives}
\author{Robert E. Kent\thanks{Current Address: 
                              Knowledge Representation Research Group,
                              Department of Computer and Information Science,
                              University of Arkansas at Little Rock
                              (rekent@ualr.edu)}}
\date{1992}

\begin{document}
   \maketitle

\begin{abstract}
The notions of formal contexts and concept lattices,
although introduced by Wille only ten years ago \cite{Wille},
already have proven to be of great utility in various applications
such as data analysis and knowledge representation.
In this paper we give arguments that Wille's original notion of formal context,
although quite appealing in its simplicity,
now should be replaced by a more semantic notion.
This new notion of formal context entails a modified  approach to concept construction.
We base our arguments for these new versions of formal context and concept construction
upon Wille's philosophical attitude with reference to the intensional aspect of concepts.
We give a brief development of the relational theory of formal contexts and concept construction,
demonstrating the equivalence of {\em concept-lattice construction\/} \cite{Wille}
with the well-known {\em completion by cuts\/} \cite{MacNeille}.
Generalization and abstraction of these formal contexts offers a powerful approach to knowledge representation.
\end{abstract}

   \tableofcontents

\newpage
\section*{Introduction}

Classification of knowledge into ordered systems
has played an important role in the history of science from the very beginning.
More recently,
classification has become quite important
in computer science in general (knowledge representation and databases)
and programming languages in particular.
The fundamental task in conceptual classification
is the search for conceptual structures (classes) naturally inherent in the problem domain
and the construction of class hierarchies.
These class hierarchies tend to be of two kinds:
generalization/specialization inheritance hierarchies
and
whole/part containment hierarchies.
Both linear and network representations of knowledge have been used for classification.
Such knowledge representation systems can have
both a structural, taxonomic aspect and an assertional, logical aspect.
Although this paper emphasizes the structural aspect,
the assertional aspect also can be represented.
Taxonomies of structured conceptual descriptions
originated in the KL-ONE knowledge representation system of Woods and others \cite{Woods}.
As Woods has pointed out,
however,
the semantics of the conceptual structures in KL-ONE and related systems is not totally clear.
Although many researchers have identified concepts with the notion of predicate in first-order logic,
Woods argues for the need to represent intensional concepts.
According to Woods,
such intensions cannot be represented in first-order logic,
and cannot be thought of as the classes of traditional knowledge representation systems.
Although Woods uses the formal notion of abstract conceptual description
as a means to logically represent both the intensional and the extensional aspects of concepts,
the simpler notion of a formal context
as championed by Rudolf Wille \cite{Wille}
is a more elegant alternative.
In this paper we give arguments that Wille's original notion of formal context,
although quite appealing in its simplicity and elegance,
now should be replaced by a more semantic notion.

The basic constituents in conceptual classification and knowledge representation
are
entities or objects corresponding to real-world objects,
and ways of describing these in terms of attributes or properties.
In programming languages attributes represent the data and operations of a data type.
The relationship between entities and attributes is a {\em has\/} relationship called a context.
Formal concept analysis,
a new approach to classification and knowledge representation initiated by Wille,
starts with the primitive notion of a formal context.
A {\em formal context\/} is a triple $\triple{X_0}{X_1}{\mu}$
consisting of two sets $X_0$ and $X_1$ and a binary relation $\mu \subseteq \product{X_0}{X_1}$
between $X_0$ and $X_1$.
Intuitively, the elements of $X_0$ are thought of as {\em entities\/} or {\em objects\/},
the elements of $X_1$ are thought of as {\em properties\/}, {\em characteristics\/} or {\em attributes\/}
that the entities might have,
and $x_0{\mu}x_1$ asserts that ``the entity $x_0$ {\em has\/} the attribute $x_1$.''
One should take note of the strict segregation
between entities on the one hand and attributes on the other.
From an extensional point-of-view
a formal context $\triple{X_0}{X_1}{\mu}$
is a base set of entities $X_0$
with an indexed collection of subsets $\{ {\mu}x_1 \mid x_1 \memberof X_1 \}$.
Then a second formal context $\triple{X_1}{X_2}{\nu}$
would be extensionally interpreted as an indexed collection of indexed collections of subsets of $X_0$:
$\{ \{ {\mu}x_1 \mid x_1 \memberof {\nu}x_2 \} \mid x_2 \memberof X_2 \}$.
Obviously,
higher-order types are implicitly represented here.
Simple relational composition of formal contexts $\triple{X_0}{X_2}{{\mu}\circ{\nu}}$
corresponds extensionally to indexed unions:
$\{ {({\mu}\circ{\nu})}x_2 \mid x_2 \memberof X_2 \} = \bigcup_{x_1 \in {\nu}x_2} {\mu}x_1$.
Relational implication $\triple{X_1}{X_2}{{\mu}\tensorimplysource{\nu}}$,
of two formal contexts $\triple{X_0}{X_1}{\mu}$ and $\triple{X_0}{X_2}{\nu}$
over the same base set of entities $X_0$,
extensionally indexes extensional containment:
${({\mu}\tensorimplysource{\nu})}x_2 = \{ x_1 \mid {\mu}x_1 \subseteq {\nu}x_2 \}$.

As Wille has explained,
formal concept analysis is based upon the understanding that a concept is a unit of thought
consisting of two parts:
its extension and its intension.
Within a certain restricted context or scope
(a type-like notion is implicit here),
the extent of a concept is a subset $\phi \memberof \powerof{X_0}$ consisting of all entities or objects belonging to the concept
--- you as an individual person belong to the concept `living person',
whereas
the intent of a concept is a subset $\psi \memberof \powerof{X_1}$ which includes all attributes or properties shared by the entities
--- all `living persons' share the attribute `can breathe'.
A concept of a given context will consist of an extent/intent pair $({\phi},{\psi})$.
These conceptual structures of Wille,
called formal concepts,
start to address the problems mentioned by Woods.
The notion of a formal concept is a valuable first step
toward a mathematical representation for the class concept from certain other classification domains
--- object-oriented programming and knowledge representation systems.
Class hierarchies are formally represented by concept lattices.
The foundation for formal concept analysis has relied in the past upon
a {\em set-theoretic\/} model of conceptual structures.
This model has been applied to both data analysis and knowledge representation.
We argue here that an enriched {\em order-theoretic\/} model for conceptual structures
provides for an improved foundation for formal concept analysis and knowledge representation.

Of central importance in concept construction are two {\em derivation operators\/}
which define the notion of ``sharing'' or ``commonality''.
For any subsets $\phi \memberof \powerof{X_0}$ and $\psi \memberof \powerof{X_1}$,
we define
\begin{center}
	$\begin{array}{l}
		\derivedir{\phi}{\mu} \define \{ x_1 \in X_1 \mid x_0{\mu}x_1 \mbox{ for all } x_0 \in \phi \} \\
		\deriveinv{\psi}{\mu} \define \{ x_0 \in X_0 \mid x_0{\mu}x_1 \mbox{ for all } x_1 \in \psi \}
	 \end{array}$
\end{center}
To demand that a concept $({\phi},{\psi})$ be determined by its extent and its intent
means
that the intent should contain precisely those attributes shared by all entities in the extent
$\derivedir{\phi}{\mu} = \psi$,
and vice-versa,
that the extent should contain precisely those entities sharing all attributes in the intent
$\phi = \deriveinv{\psi}{\mu}$.
Concepts are ordered by generalization/specialization:
one concept is more specialized (and less general) than another
$(\phi,\psi) \leq_{\rm CL} (\phi',\psi')$
when its intent contains the other's intent
$\psi \supseteq \psi'$,
or equivalently,
when the opposite ordering on extents occurs
$\phi \subseteq \phi'$.
Concepts with this generalization/specialization ordering
form a concept hierarchy for the context.
The concept hierarchy is a complete lattice ${\bf CL}\triple{X_0}{X_1}{\mu}$
called the {\em concept lattice\/} of $\triple{X_0}{X_1}{\mu}$.
The meets and joins in ${\bf CL}\triple{X_0}{X_1}{\mu}$ can be described as follows:
\begin{center}
	$\begin{array}{r@{\;=\;}l}
		\bigwedge_{k \in K} (\phi_k,\psi_k)
		& \left( \bigcap_{k \in K} \phi_k, \closure{(\bigcup_{k \in K} \psi_k)}{\mu} \right)
		\\
		\bigvee_{k \in K} (\phi_k,\psi_k)
		& \left( \closure{(\bigcup_{k \in K} \phi_k)}{\mu}, \bigcap_{k \in K} \psi_k \right)
	 \end{array}$
\end{center}
The join of a collection of concepts
represents what the concepts have in `common' or `share',
and the top of the concept hierarchy represents all entities
(the universal concept).

The information presented in Table~\ref{planets}
and originally described in \cite{Wille}
gives a limited context for the planets of our solar system.
The entities $X_0 = \{ {\rm Me},{\rm V},{\rm {\rm E}},{\rm Ma},{\rm J},{\rm S},{\rm U},{\rm N},{\rm P} \}$ are the planets
and the attributes $X_1 = \{ {\rm ss},{\rm sm},{\rm sl},{\rm dn},{\rm df},{\rm my},{\rm mn} \}$ are the seven scaled properties
relating to size, distance from the sun, and existence of moons,
with abbreviations
\scriptsize
\begin{center}
\begin{tabular}{ccc}
	\begin{tabular}{|l@{\,\,---\,\,\,}l|} \hline
		\multicolumn{2}{|c|}{\normalsize{\bf entities}\scriptsize} \\ \hline
		Me & Mercury  \\
		V  & Venus   \\
		E  & Earth   \\
		Ma & Mars    \\
		J  & Jupiter \\
		S  & Saturn  \\
		U  & Uranus  \\
		N  & Neptune \\
		P  & Pluto   \\ \hline
	\end{tabular}
	&
	\normalsize and \scriptsize
	&
	\begin{tabular}{|l@{\,\,---\,\,\,}l|} \hline
		\multicolumn{2}{|c|}{\normalsize{\bf attributes}\scriptsize} \\ \hline
		ss & size{\bf :}small    \\
		sm & size{\bf :}medium   \\
		sl & size{\bf :}large    \\
		dn & distance{\bf :}near \\
		df & distance{\bf :}far  \\
		my & moon{\bf :}yes      \\
		mn & moon{\bf :}no       \\ \hline
	\end{tabular}
\end{tabular}
\end{center}
\normalsize
The table itself represents the {\em has\/} relationship $\mu \subset \product{X_0}{X_1}$.
The fact $x_0{\mu}x_1$ that the $x_0$th object has the $x_1$th attribute
is indicated by a `$\times$' in the ${x_0}{x_1}$th entry in Table~\ref{planets}.
\scriptsize
\begin{table}
\begin{center}
	\begin{tabular}{|l|ccccccc|} \hline
		   & ss & sm & sl & dn & df & my & mn     \\ \hline
		Me & $\times$ &&& $\times$ &&& $\times$   \\
		V  & $\times$ &&& $\times$ &&& $\times$   \\
		E  & $\times$ &&& $\times$ && $\times$ &  \\
		Ma & $\times$ &&& $\times$ && $\times$ &  \\
		J  &&& $\times$ && $\times$ & $\times$ &  \\
		S  &&& $\times$ && $\times$ & $\times$ &  \\
		U  && $\times$ &&& $\times$ & $\times$ &  \\
		N  && $\times$ &&& $\times$ & $\times$ &  \\
		P  & $\times$ &&&& $\times$ & $\times$ &  \\ \hline
	\end{tabular}
\end{center}
\caption{{\bf A contextual relationship for planets} \label{planets}}
\end{table}
\normalsize
The concepts for this planetary context are listed in Table~\ref{CLplanets}.
\footnotesize
\begin{table}
\begin{center}
	\begin{tabular}{|c|c|c|}
		\hline
		{\bf concept}
			& {\bf extent}
				& {\bf intent} \\
		{\bf description}
			& $\phi$
				& $\psi$ \\
		\hline\hline
		``everything''
			& $X_0$
				& $\emptyset$ \\
		``with moon''
			& $\{ {\rm E}, {\rm Ma}, {\rm J}, {\rm S}, {\rm U}, {\rm N}, {\rm P} \}$
				& $\{ {\rm my} \}$ \\
		``small size''
			& $\{ {\rm Me}, {\rm V}, {\rm E}, {\rm Ma}, {\rm P} \}$
				& $\{ {\rm ss} \}$ \\
		``small with moon''
			& $\{ {\rm E}, {\rm Ma}, {\rm P} \}$
				& $\{ {\rm ss}, {\rm my} \}$ \\
		``far''
			& $\{ {\rm J}, {\rm S}, {\rm U}, {\rm N}, {\rm P} \}$
				& $\{ {\rm df}, {\rm my} \}$ \\
		``near''
			& $\{ {\rm Me}, {\rm V}, {\rm E}, {\rm Ma} \}$
				& $\{ {\rm ss}, {\rm dn} \}$ \\
		``Plutoness''
			& $\{ {\rm P} \}$
				& $\{ {\rm ss}, {\rm df}, {\rm my} \}$ \\
		``medium size''
			& $\{ {\rm U}, {\rm N} \}$
				& $\{ {\rm sm}, {\rm df}, {\rm my} \}$ \\
		``large size''
			& $\{ {\rm J}, {\rm S} \}$
				& $\{ {\rm sl}, {\rm df}, {\rm my} \}$ \\
		``near with moon''
			& $\{ {\rm E}, {\rm Ma} \}$
				& $\{ {\rm ss}, {\rm dn}, {\rm my} \}$ \\
		``moonless''
			& $\{ {\rm Me}, {\rm V} \}$
				& $\{ {\rm ss}, {\rm dn}, {\rm mn} \}$ \\
		``nothing''
			& $\emptyset$
				& $X_1$ \\
		\hline
	\end{tabular}
\end{center}
\caption{{\bf Concept lattice ${\bf CL}\triple{X_0}{X_1}{\mu}$ for the planetary relationship} \label{CLplanets}}
\end{table}
\normalsize

Entities generate concepts.
There is a function
$\morphism{X_0}{\hat{\imath}_0}{{\bf CL}\triple{X_0}{X_1}{\mu}}$
called the generator function
which maps each entity $x_0 \memberof X_0$
to its associated concept
$\hat{\imath}_0(x_0) \define (\closure{(x_0)}{\mu},\derivedir{(x_0)}{\mu})$.
Similarly, attributes generate concepts
by means of a generator function
$\morphism{X_1}{\hat{\imath}_1}{{\bf CL}\triple{X_0}{X_1}{\mu}}$
which maps attributes $x_1 \memberof X_1$
to their associated concept
$\hat{\imath}_1(x_1) \define (\deriveinv{(x_1)}{\mu},\closure{(x_1)}{\mu})$.
These functions are dense:
the entity generator function is join-dense in the concept lattice,
since $\derivedir{\phi}{\mu} = \bigcap_{x_0 \in \phi} \derivedir{(x_0)}{\mu}$
for any subset $\phi \memberof \powerof{X_0}$;
and the attribute generator function is meet-dense in the concept lattice,
since $\deriveinv{\psi}{\mu} = \bigcap_{x_1 \in \psi} \deriveinv{(x_1)}{\mu}$
for any subset $\psi \memberof \powerof{X_1}$.
The concept lattice contains all of the information in the formal context:
the has relation $\mu$ of the formal context can be expressed
in terms of the generator maps plus the concept order
\[ x_0{\mu}x_1 \mbox{  iff  } \hat{\imath}_0(x_0) \leq_{\rm CL} \hat{\imath}_1(x_1). \]

\begin{Theorem}
	{\bf Basic Theorem for Concept Lattices:} \cite{Wille}
	For any formal context $\triple{X_0}{X_1}{\mu}$ the set of concepts
	${\bf CL}\triple{X_0}{X_1}{\mu}$ with the generalization/specialization order
	forms a complete lattice.
	Any complete lattice $L$ is isomorphic to ${\bf CL}\triple{X_0}{X_1}{\mu}$
	iff
	there are two maps $X_0 \stackrel{f_0}{\longrightarrow} L \stackrel{f_1}{\longleftarrow} X_1$
	where $f_0$ is join-dense in L, $f_1$ is meet-dense in L,
	and $x_0{\mu}x_1 \mbox{  iff  } f_0(x_0) \leq_L f_1(x_1)$.
\end{Theorem}

These notions of
formal context, concept, concept construction via derivation operators, and concept lattice
are comparable to the notions of
order, closed subset, upper/lower operators, and Dedekind-MacNeille completion of the order \cite{MacNeille}.
The Dedekind-MacNeille completion generalizes Dedekind's construction of the real number system
from the rational numbers via cuts.
Indeed, Dedekind's real number construction is the Dedekind-MacNeille completion of the rational number system order.
By the same token,
the Dedekind-MacNeille completion of any order is the concept lattice construction
of the formal context consisting of the ordering relationship,
where the elements from the underlying set of the order
are interpreted as both entities and attributes
(a simple, yet important, example indicating the interchangeability between entities and attributes,
and de-emphasizing their distinctness).
A more abstract version of aspects of the Dedekind-MacNeille completion
was essential in the development of a classical relational logic \cite{Kent}.

Does the Dedekind-MacNeille completion play a strictly subordinate role to the concept lattice construction,
or are the two constructions equivalent in some sense?
This paper provides an affirmative answer to the latter question
by showing how Wille's concept lattice construction can be viewed
as (a local component of) the Dedekind-MacNeille completion of a distributed order
--- the two constructions are equivalent approaches to classification.
The philosophy behind this approach is the view that
a formal context is a kind of constraint which qualifies parallel control/data flow.
All this development centers around the central adjointness of relational logic between composition and implication.

The first section explains with examples the meaning of collective entities and collective attributes.
In the second section we discover the hidden relationship opposite to the original {\em has\/} relationship.
Section three defines our new version of formal context, introduces contextual closure, and describes formal concepts from this new standpoint.
In section four we explain a new notion of order-theoretic sum which centralizes our distributed version of formal contexts,
and we state and prove the Equivalence Theorem which relates the concept lattice construction with Dedekind-MacNeille completion.
Finally, in section five we indicate some areas of new research.


\section{Conceptual Collectives} \label{ConceptualCollectives}

The {\bf first} step that we take in the analysis of concept construction
is the observation, already made by Wille, that a formal context,
although defined {\em a priori\/} in terms of sets and relations,
has order relationships on entities and attributes
induced by the corresponding generator map into the concept lattice.

Part of the basic theorem for concept lattices
states that the binary relation of the original formal context
can be recovered
via the two generator functions plus the concept lattice order:
$x_0{\mu}x_1 \mbox{ iff } \hat{\imath}_0(x_0) \leq_{\rm CL} \hat{\imath}_1(x_1)$.
This same approach can be used to define order relations on both source and target sets
\cite{Wille}:
\begin{enumerate}
	\item ${\cal X}_0^{\mu} \define \pair{X_0}{\leq_0^{\mu}}$,
		where the order relation ${\leq_0^{\mu}}$ is defined by
		$x'_0 \leq_0^{\mu} x_0$
		iff
		$\hat{\imath}_0(x'_0) \leq_{\rm CL} \hat{\imath}_0(x_0)$
		iff
		$x'_0{\mu} \supseteq x_0{\mu}$,
		so that ${\leq_0^{\mu}} = {\mu} \tensorimplytarget {\mu}$;
		${\cal X}_0^{\mu}$ is the largest source order for which ${\mu}$ is closed on the left.
	\item ${\cal X}_1^{\mu} \define \pair{X_1}{\leq_1^{\mu}}$,
		where the order relation ${\leq_1^{\mu}}$ is defined by
		$x_1 \leq_1^{\mu} x'_1$
		iff
		$\hat{\imath}_1(x_1) \leq_{\rm CL} \hat{\imath}_1(x'_1)$
		iff
		${\mu}x_1 \subseteq {\mu}x'_1$,
		so that ${\leq_1^{\mu}} = {\mu} \tensorimplysource {\mu}$;
		${\cal X}_1^{\mu}$ is the largest target order for which ${\mu}$ is closed on the right.
\end{enumerate}
The binary relation from the original context is a closed relation w.r.t.\ these induced orders
\[ \term{{\cal X}_0^{\mu}}{\mu}{{\cal X}_1^{\mu}}. \]
The source and target orders induced by the concept lattice of the planetary context
are displayed in Table~\ref{inducedorders}.
\scriptsize
\begin{table}
\begin{center}
\begin{tabular}{c@{\hspace{1cm}}c}
	\begin{tabular}{c}
	\begin{tabular}{|l|@{\hspace{1.5mm}}c@{\hspace{1.5mm}}c@{\hspace{1.5mm}}c@{\hspace{1.5mm}}c@{\hspace{1.5mm}}c@{\hspace{1.5mm}}c@{\hspace{1.5mm}}c@{\hspace{1.5mm}}c@{\hspace{1.5mm}}c@{\hspace{0.5mm}}|} \hline
		   & Me & V & E & Ma & J & S & U & N & P  \\ \hline
		Me & $\times$ & $\times$ &&&&&&& \\
		V  & $\times$ & $\times$ &&&&&&& \\
		E  &&& $\times$ & $\times$ &&&&& \\
		Ma &&& $\times$ & $\times$ &&&&& \\
		J  &&&&& $\times$ & $\times$ &&& \\
		S  &&&&& $\times$ & $\times$ &&& \\
		U  &&&&&&& $\times$ & $\times$ & \\
		N  &&&&&&& $\times$ & $\times$ & \\
		P  &&&&&&&&& $\times$  \\ \hline
	\end{tabular}
	\\ \\
	\normalsize {\bf Source order} \scriptsize
	\end{tabular}
	&
	\begin{tabular}{c}
	\begin{tabular}{|l|@{\hspace{1mm}}c@{\hspace{1mm}}c@{\hspace{1mm}}c@{\hspace{1mm}}c@{\hspace{1mm}}c@{\hspace{1mm}}c@{\hspace{1mm}}c@{\hspace{0.5mm}}|} \hline
		   & ss & sm & sl & dn & df & my & mn     \\ \hline
		ss & $\times$ &&&&&&                      \\
		sm && $\times$ &&& $\times$ & $\times$ &  \\
		sl &&& $\times$ && $\times$ & $\times$ &  \\
		dn & $\times$ &&& $\times$ &&&            \\
		df &&&&& $\times$ & $\times$ &            \\
		my &&&&&& $\times$ &                      \\
		mn & $\times$ &&& $\times$ &&& $\times$   \\ \hline
	\end{tabular}
	\\ \\
	\normalsize {\bf Target order} \scriptsize
	\end{tabular}
\end{tabular}
\end{center}
\caption{{\bf Induced orders for the planetary context} \label{inducedorders}}
\end{table}
\normalsize

Intuitively,
for the source set the order $\leq_0^{\mu}$ specifies implicational information
between entities:
an order relationship $x'_0 \leq_0^{\mu} x_0$
exists between two entities $x'_0$ and $x_0$
when $x'_0$ has all the attributes of $x_0$
(and possibly some others).
In a sense,
$x_0$ represents its collection of individual attributes
$x_0{\mu} \subseteq x'_0{\mu} \subseteq X_1$
--- in fact,
we can regard the entity $x_0$ as being a kind of {\em collective attribute\/} of $x'_0$.
In particular,
any entity is a collective attribute of itself.
By the same token,
for the target set the order $\leq_1^{\mu}$ specifies implicational information
between attributes:
an order relationship $x_1 \leq_1^{\mu} x'_1$
exists between two attributes $x_1$ and $x'_1$
when entities which have attribute $x_1$ also have attribute $x'_1$;
$x_1$ represents its collection of individual entities
${\mu}x_1 \subseteq {\mu}x'_1 \subseteq X_0$,
we regard the attribute $x_1$ as being a kind of {\em collective entity\/} of $x'_1$,
and any attribute is a collective entity of itself.
These dual relationships are picture in Figure~\ref{collective}.
\begin{figure}
\begin{center}
	\begin{tabular}{c@{\hspace{5mm}}c}
		\begin{tabular}{c}
			$\underbrace{\mbox{\begin{tabular}{c@{\hspace{5mm}}c}
				\raisebox{20pt}{$x'_0 \preceq_0^{\mu} x_0$}
					& \begin{picture}(50,50)(-25,-25)
						\put(0,14){\makebox(0,0){\footnotesize$x'_0{\mu}$\normalsize}}
						\put(0,-5){\makebox(0,0){\footnotesize$x_0{\mu}$\normalsize}}
						\thicklines
						\put(0,0){\circle{50}}
						\put(0,-5){\circle{20}}
						\thinlines
					\end{picture}
						\\
				individual entities
					& \begin{tabular}{c}subsets of\\individual attributes\end{tabular}
			\end{tabular}}}$
				\\
			individual entity $x'_0$ having collective attribute $x_0$
		\end{tabular}
		&
		\begin{tabular}{c}
			$\underbrace{\mbox{\begin{tabular}{c@{\hspace{5mm}}c}
					\begin{picture}(50,50)(-25,-25)
						\put(0,14){\makebox(0,0){\footnotesize${\mu}x'_1$\normalsize}}
						\put(0,-5){\makebox(0,0){\footnotesize${\mu}x_1$\normalsize}}
						\thicklines
						\put(0,0){\circle{50}}
						\put(0,-5){\circle{20}}
						\thinlines
					\end{picture}
						& \raisebox{20pt}{$x_1 \preceq_1^{\mu} x'_1$}
							\\
					\begin{tabular}{c}subsets of\\individual entities\end{tabular}
						& individual attributes
			\end{tabular}}}$
				\\
			collective entity $x_1$ having individual attribute $x'_1$
		\end{tabular}
	\end{tabular}
\end{center}
\caption{{\bf collective entities, collective attributes} \label{collective}}
\end{figure}
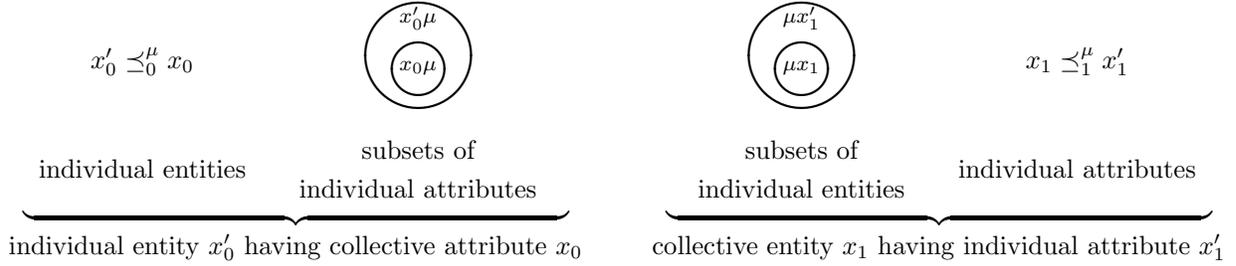
Here we see more evidence of the interchangeability of entities and attributes,
arguing for a certain kind of inherent blending or integration of the two notions.
Ultimately
we will argue for the total blending or integration of entities and attributes,
but in a very structured fashion which allows for a locally relative distinction.

Since the order ${\cal X}_0^{\mu}$ is a legitimate relationship between entities and (collective) attributes,
we can define direct and inverse derivation along the identity relationship ${\cal X}_0^{\mu}$.
These are identical with the upper and lower operators on ${\cal X}_0^{\mu}$.
Suppose $\phi \subseteq X_0$ is a subset of entities.
What is the meaning of the upper operator applied to $\phi$?
The upper operator applied to $\phi$ returns the closed-above subset
\begin{center}
	$\begin{array}{r@{\:=\;}l}
		\upper{\phi}{{\cal X}_0^{\mu}}
		& \derivedir{\phi}{{\cal X}_0^{\mu}} \\
		& \{ x'_0 \memberof X_0 \mid x_0 \leq_0^{\mu} x'_0 \mbox{ for all } x_0 \memberof \phi \} \\
		& \{ x'_0 \memberof X_0 \mid \hat{\imath}_0(x_0) \leq_{\rm CL} \hat{\imath}_0(x'_0) \mbox{ for all } x_0 \memberof \phi \} \\
		& \{ x'_0 \memberof X_0 \mid x_0{\mu} \supseteq x'_0{\mu} \mbox{ for all } x_0 \memberof \phi \},
	 \end{array}$
\end{center}
which consists of all collective attributes of all entities in $\phi$.
An entity in $\upper{\phi}{{\cal X}_0^{\mu}}$
represents a certain `type of commonality' for all the entities in $\phi$.
We wish to emphasize the obvious analogy with the definition of the direct derivation $\derivedir{\phi}{\mu}$ along $\mu$,
the main and only distinction being that
direct derivation $\derivedir{\phi}{\mu}$
returns all {\em individual attributes\/} of all entities in $\phi$,
whereas the upper operator $\upper{\phi}{{\cal X}_0^{\mu}}$
returns all {\em collective attributes\/} of all entities in $\phi$.
A dual discussion can be given for inverse derivation where existence of collective entities in the target set is observed.
Formal contexts and their various constituents are illustrated in Figure~\ref{formalcontexts}.
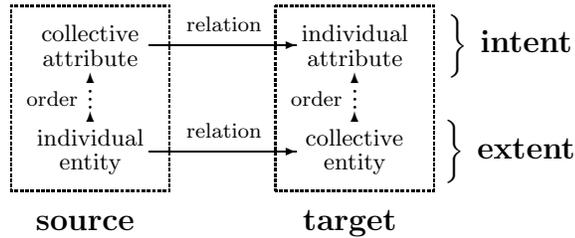
\begin{figure}
\begin{center}
	\begin{picture}(200,80)(-40,-46)
		\put(0,20){\makebox(0,0){\small\shortstack{collective\\attribute}\normalsize}}
		\put(0,-20){\makebox(0,0){\small\shortstack{individual\\entity}\normalsize}}
		\put(0,-47){\makebox(0,0){\large {\bf source} \normalsize}}
		\put(100,20){\makebox(0,0){\small\shortstack{individual\\attribute}\normalsize}}
		\put(100,-20){\makebox(0,0){\small\shortstack{collective\\entity}\normalsize}}
		\put(100,-47){\makebox(0,0){\large {\bf target} \normalsize}}
		\put(160,-20){\makebox(0,0){$\left.\rule{0cm}{5mm}\right\}$\hspace{5pt}\large {\bf extent} \normalsize}}
		\put(160,20){\makebox(0,0){$\left.\rule{0cm}{5mm}\right\}$\hspace{5pt}\large {\bf intent} \normalsize}}
		\put(52,28){\makebox(0,0){\footnotesize relation \normalsize}}
		\put(52,-12){\makebox(0,0){\footnotesize relation \normalsize}}
		\put(-13,0){\makebox(0,0){\footnotesize order \normalsize}}
		\put(87,0){\makebox(0,0){\footnotesize order \normalsize}}
		\put(22,20){\vector(1,0){56}}
		\put(22,-20){\vector(1,0){56}}
		\put(0,10){\vector(0,1){0}}
		\multiput(0,-3.5)(0,3){3}{\makebox(0,0){$\cdot$}}
		\put(0,-5){\vector(0,1){0}}
		\put(100,10){\vector(0,1){0}}
		\multiput(100,-3.5)(0,3){3}{\makebox(0,0){$\cdot$}}
		\put(100,-5){\vector(0,1){0}}
		\put(-30,-35){\dashbox{1}(60,70){}}
		\put(70,-35){\dashbox{1}(60,70){}}
	\end{picture}
\end{center}
\caption{{\bf Formal contexts} \label{formalcontexts}}
\end{figure}

To summarize the discussion above,
given any relation $\term{{\cal X}_0}{\mu}{{\cal X}_1}$
used to represent concepts in knowledge representation,
intuitively each element $x_0 \memberof X_0$ is
both an individual entity following the standard interpretation
and a collective attribute representing the collection
$x_0{\mu} \memberof \powerof{{\cal X}_1}$ of individual attributes,
and dually,
each element $x_1 \memberof X_1$ is
both an individual attribute following the standard interpretation
and a collective entity representing the collection
${\mu}x_1 \memberof \powerof{{\cal X}_0^{\rm op}}$ of individual entities.
We will use this point of view
in order to understand the appropriate course of action for concept construction
when order information is specified {\em a priori\/} for both source (entities) and target (attribute) sets.

\begin{table}
\begin{center}
	\begin{tabular}{|c|ccccc|} \hline
			& \shortstack{{\bf individual}\\{\bf entity}} & {\bf order} & \shortstack{{\bf collective}\\{\bf attribute}} & {\bf relation} & \shortstack{{\bf individual}\\{\bf attribute}}
			\\ \hline\hline
		knowledge representation
			& cat & is-a & animal & requires & oxygen
			\\ \hline
		linguistics
			& ? & is-a & ? & ? & ?
			\\ \hline
		planets
			& Earth & equiv-to & Mars & has & size:small
			\\ \hline
		database
			& Smith  & is-a & engineer & works-for & Aerospace, Inc.
			\\ \hline
		generic database
			& employee & works-for & company & located-in & city
			\\ \hline
	\end{tabular}
\end{center}
\caption{{\bf Examples of collective attributes} \label{collatts}}
\end{table}

\begin{table}
\begin{center}
	\begin{tabular}{|c|ccccc|} \hline
			& \shortstack{{\bf individual}\\{\bf entity}} & {\bf relation} & \shortstack{{\bf collective}\\{\bf entity}} & {\bf order} & \shortstack{{\bf individual}\\{\bf attribute}}
			\\ \hline\hline
		Roget's thesaurus
			& word-string & occurs-in & RIT-category       & sub-type & RIT-class \\
			& ``toast''   &           & {\bf 324}: Cooking &          & {\bf Three}: Physics
			\\ \hline
		planets
			& Pluto & has & distance:far & implies & moon:yes
			\\ \hline
			& Felix & member & cat & sub-type & mammal
			\\ \hline
	\end{tabular}
\end{center}
\caption{{\bf Examples of collective entities} \label{collents}}
\end{table}


\section{Concept Construction}

The {\bf second} step is crucial!
We use Wille's philosophical position regarding the extension/intension duality of concepts,
and argue that the entity and attribute orders are themselves local relationships
which should be used simultaneously with the original relationship of the context.
The argument here centers around the viewpoint
that entities can be seen as collective attributes
and dually
that attributes can be seen as collective entities
--- whence the title of the paper.

Formal contexts have an order-theoretic nature,
in the sense that
at least an implicit order exists on both source set (entities) and target set (attributes).
We can respect this observation 
by defining a formal context {\em a priori\/} in terms of orders and order-closed relations,
effectively changing from the set-theoretic to the order-theoretic realm.
This order-theoretic realm
replaces sets with orders and replaces ordinary relations with closed relations
(other enriched realms will be considered in a subsequent paper,
where a more formal and abstract analysis is given).
Let us use the point of view espoused in the first subsection above
in order to understand the appropriate course of action for concept construction
when order information is specified {\em a priori\/} for both source (entities) and target (attribute) sets.
Let the closed relation $\term{{\cal X}_0}{\mu}{{\cal X}_1}$ represent
a formal context in the enriched order-theoretic realm.
The order information $x'_0 \preceq_{{\cal X}_0} x_0$
specified {\em a priori\/} is interpreted as
``the entity $x_0$ is a collective attribute of $x'_0$''.
By closure of the relation ${\mu}$ at the source order ${\cal X}_0$,
any ${\mu}$-attribute of $x_0$ is an ${\mu}$-attribute of $x'_0$.
Now given a subset $\phi \subseteq X_0$,
when computing the common shared attributes of elements of $\phi$ during concept construction,
it seems appropriate to consider not only application of the direct derivation operator
\[ \longmorphism{\powerof{{\cal X}_0^{\rm op}}}{\derivedir{(\,)}{\mu}}{{(\powerof{{\cal X}_1})}^{\rm op}} \]
getting the order filter
$\derivedir{\phi}{\mu} \in \powerof{{\cal X}_1}$
of all shared individual attributes,
but also application of the upper operator
\[ \longmorphism{\powerof{{\cal X}_0^{\rm op}}}{\upper{(\,)}{{\cal X}_0}}{{(\powerof{{\cal X}_0})}^{\rm op}} \]
getting the order filter
$\upper{\phi}{{\cal X}_0} \in \powerof{{\cal X}_0}$
of all shared collective attributes.
This pair of order filters satisfies the filter assertion
\[ \upper{\phi}{{\cal X}_0} \circ {\mu} \preceq \derivedir{\phi}{\mu} \]
since
$\phi \circ \upper{\phi}{{\cal X}_0} \circ {\mu}
 = \phi \circ (\phi \tensorimplysource {\cal X}_0) \circ {\mu}
 \preceq {\cal X}_0 \circ {\mu}
 = {\mu}$.

To recapitulate,
if we start with a single order ideal in $\powerof{{\cal X}_0^{\rm op}}$,
the direct phase of concept construction returns two 
assertionally constrained order filters,
one in $\powerof{{\cal X}_0}$ and one in $\powerof{{\cal X}_1}$.
Such constrained pairs of order filters
provide a necessary structural constraint on the {\em intensional aspect\/} of concepts:
the {\em intent\/} of a concept is a pair $(\psi_0,\psi_1)$,
an order filter of collective attributes $\psi_0 \in \powerof{{\cal X}_0}$
and an order filter of individual attributes $\psi_1 \in \powerof{{\cal X}_1}$
subject to the filter assertion
\[ \psi_0 \circ {\mu} \preceq \psi_1. \]
The need of this assertional constraint for ``type summability''
is discussed below in the order-theoretic realm.
It places a restriction upon filter pairs,
allowing only certain admissible pairs,
and is described by the slogan
\begin{center}
	\fbox{\begin{tabular}{c}
          	The (image of the) collective component \\
			is contained in the individual component.
	      \end{tabular}}
\end{center}
Continuing the argument above,
in the inverse phase of concept construction,
we start from the intent
--- the common attributes of a concept.
For this inverse phase,
since there are (at least) three relationships
$\mu$, $\preceq_{{\cal X}_0}$ and $\preceq_{{\cal X}_1}$,
there are (at least) three relevant operators:
\begin{enumerate}
	\item the inverse derivation operator
		\[ \longmorphism{{(\powerof{{\cal X}_1})}^{\rm op}}{\deriveinv{(\,)}{\mu}}{\powerof{{\cal X}_0^{\rm op}}} \]
		which when applied to the order filter $\psi_1 \in \powerof{{\cal X}_1}$ of individual attributes 
		returns the order ideal $\deriveinv{(\psi_1)}{\mu} \in \powerof{{\cal X}_0^{\rm op}}$
		of all individual entities which share all of the individual attributes in $\psi_1$,
	\item the lower operator
		\[ \longmorphism{{(\powerof{{\cal X}_1})}^{\rm op}}{\looer{(\,)}{{\cal X}_1}}{\powerof{{\cal X}_1^{\rm op}}} \]
		which when applied to the order filter $\psi_1 \in \powerof{{\cal X}_1}$ of individual attributes 
		returns the order ideal $\looer{(\psi_1)}{{\cal X}_1} \in \powerof{{\cal X}_1^{\rm op}}$
		of all collective entities which share all of the individual attributes in $\psi_1$,
		and
	\item the lower operator
		\[ \longmorphism{{(\powerof{{\cal X}_0})}^{\rm op}}{\looer{(\,)}{{\cal X}_0}}{\powerof{{\cal X}_0^{\rm op}}} \]
		which when applied to the order filter $\psi_0 \in \powerof{{\cal X}_0}$ of collective attributes 
		returns the order ideal $\looer{(\psi_0)}{{\cal X}_0} \in \powerof{{\cal X}_0^{\rm op}}$
		of all individual entities which share all of the collective attributes in $\psi_0$.
\end{enumerate}
Since we are again constructing commonality,
it is appropriate to take the meet 
$\looer{(\psi_0)}{{\cal X}_0} \wedge \deriveinv{(\psi_1)}{\mu}$
of the two order ideals in $\powerof{{\cal X}_0^{\rm op}}$;
this consists of all entities in common with
both collective attributes in $\psi_0$ and individual attributes in $\psi_1$.
We end up with the pair of order ideals
$\pair{\looer{(\psi_0)}{{\cal X}_0} \wedge \deriveinv{(\psi_1)}{\mu}}{\looer{(\psi_1)}{{\cal X}_1}}$.

{\bf SOMETHING IS WRONG!}
--- this pair does {\bf NOT} necessarily satisfy the ideal assertion
\[ {\mu} \circ \looer{(\psi_1)}{{\cal X}_1} \;\;\preceq\;\; \looer{(\psi_0)}{{\cal X}_0} \wedge \deriveinv{(\psi_1)}{\mu} \]
which is the codification of the admissibility slogan
for the {\em extensional aspect\/} of concepts,
that
``the (image of the) collective component is contained in the individual component''.
We note that ``half'' of this constraint does hold:
\[ {\mu} \circ \looer{(\psi_1)}{{\cal X}_1} \;\;\preceq\;\; \deriveinv{(\psi_1)}{\mu} \]
since
$\mu \circ \looer{(\psi_1)}{{\cal X}_1} \circ \psi_1
 = \mu \circ ({\cal X}_1 \tensorimplytarget \psi_1) \circ \psi_1
 \preceq \mu \circ {\cal X}_1
 = \mu$.
In order to satisfy the full assertional constraint,
we need an appropriate factor (order ideal) $\alpha$
which will restrict the collective component of the extent
in correspondence with
the restriction of the individual component of the extent by $\looer{(\psi_0)}{{\cal X}_0}$.
Then the full assertional constraint would be
\[ {\mu} \circ (\alpha \wedge \looer{(\psi_1)}{{\cal X}_1})
   \;\;\preceq\;\;
   \looer{(\psi_0)}{{\cal X}_0} \wedge \deriveinv{(\psi_1)}{\mu} \]
For this to hold,
a sufficient condition on the factor $\alpha$
is the partial (half) assertional constraint
\[ {\mu} \circ \alpha \;\;\preceq\;\; \looer{(\psi_0)}{{\cal X}_0} \]
The maximal order ideal satisfying this constraint is
$\alpha \define \mu \tensorimplysource \looer{(\psi_0)}{{\cal X}_0}
 = \mu \tensorimplysource ({\cal X}_0 \tensorimplytarget \psi_0)
 = (\mu \tensorimplysource {\cal X}_0) \tensorimplytarget \psi_0
 = \deriveinv{(\psi_0)}{(\mu \tensorimplysource {\cal X}_0)}$.
We interpret this as the inverse derivation of the order filter of collective attributes $\psi_0$
along the source negation
$\term{{\cal X}_1}{\mu \tensorimplysource {\cal X}_0}{{\cal X}_0}$.
The source negation is the largest relation
$\term{{\cal X}_1}{\nu}{{\cal X}_0}$
that is opposite to $\mu$
and
satisfies the partial asymmetric orthogonal constraint
$\mu \circ \nu \preceq {\cal X}_0$.

So the error above was an {\bf ERROR OF OMISSION}
--- there is a {\em hidden relationship\/} in the opposite direction to $\mu$ that also must be considered.
This hidden relationship provides for a fourth operator active in the inverse phase of concept construction above.
Actually,
in order for derivation to work correctly
in both the direct {\em and\/} inverse phases of concept construction,
we must use a relation no larger than the {\em negation\/} of $\mu$ \cite{Kent},
the relation
$\term{{\cal X}_1}{\neg{\mu}}{{\cal X}_0}$
defined by
\[ \neg{\mu} \define (\mu \tensorimplysource {\cal X}_0) \wedge ({\cal X}_1 \tensorimplytarget \mu) \]
This is the largest relation
$\term{{\cal X}_1}{\nu}{{\cal X}_0}$
which is opposite to $\mu$
and
satisfies the full symmetric orthogonal constraints
\[
 \mu \circ \nu \preceq {\cal X}_0
 \hspace{1cm} \mbox{and} \hspace{1cm}
 \nu \circ \mu \preceq {\cal X}_1
\]
Note that $x_0{\mu}x_1$ implies
${\downarrow}x_0 \subseteq {\mu}x_1$ and ${\uparrow}x_1 \subseteq x_0{\mu}$.
Since
$\mu \tensorimplysource {\cal X}_0
 = \{ (x_1,x_0) \mid x_1 \memberof X_1, x_0 \memberof X_0, {\mu}x_1 \subseteq {\downarrow}x_0 \}
 = \{ (x_1,x_0) \mid x_1 \memberof X_1, x_0 \memberof X_0, \deriveinv{(x_1)}{\mu} \preceq \looer{(x_0)}{{\cal X}_0} \}$
and
${\cal X}_1 \tensorimplytarget \mu
 = \{ (x_1,x_0) \mid x_1 \memberof X_1, x_0 \memberof X_0, x_0{\mu} \subseteq {\uparrow}x_1 \}
 = \{ (x_1,x_0) \mid x_1 \memberof X_1, x_0 \memberof X_0, \derivedir{(x_0)}{\mu} \preceq \upper{(x_1)}{{\cal X}_1} \}$,
negation is a kind of contrapositive of $\mu$
defined pointwise by
\[ \neg{\mu}
   \define \{ (x_1,x_0) \mid x_1 \memberof X_1, x_0 \memberof X_0, 
              {\mu}x_1 \subseteq {\downarrow}x_0 \;\mbox{and}\; x_0{\mu} \subseteq {\uparrow}x_1 \}. \]
Negations of some special relations are described in Table~\ref{negations}.
\begin{table}
\begin{center}
	$\begin{array}{|cc|r@{\;=\;}l|}
		\hline
		\multicolumn{2}{|c|}{\mbox{\bf relation}}
			& \multicolumn{2}{|c|}{\mbox{\bf negation}}
			\\
		\multicolumn{2}{|c|}{\term{{\cal X}_0}{\mu}{{\cal X}_1}}
			& \multicolumn{2}{|c|}{\term{{\cal X}_1}{{\neg}{\mu}}{{\cal X}_0}}
					\\ \hline\hline
		{\bf join} & \term{{\cal X}_0}{{\mu}{\vee}{\nu}}{{\cal X}_1}
			& {\neg}(\mu \vee \nu) & {\neg}{\mu} \wedge {\neg}{\nu}
				\\ \hline
		{\bf bottom} & \term{{\cal X}_0}{\bot}{{\cal X}_1}
			& {\neg}{\bot} & {\top}
				\\ \hline
		{\bf identity} & \term{{\cal X}}{{\cal X}}{{\cal X}}
			& \hspace{2.65cm}{\neg}{{\cal X}} & {\cal X}
				\\ \hline
		{\bf complement} & \term{{\cal X}}{\not\geq}{{\cal X}}
			& \multicolumn{2}{c|}{\begin{array}{r@{\;=\;}l}
				{\neg}{\not\geq} & \{ (x,x') \mid (\forall_{y \in X})\; x \leq_{\cal X} y \;\mbox{or}\; y \leq_{\cal X} x' \}
				\\
				& \{ (x,x') \mid X = {\uparrow}x \cup {\downarrow}x' \}
			  \end{array}}
				\\ \hline
		{\bf ideal} & \term{{\cal X}}{\phi}{{\bf 1}}
			& \multicolumn{2}{c|}{{\neg}{\phi} \;=\; \phi \tensorimplysource {\cal X} = \upper{\phi}{{\cal X}}}
				\\ \hline
		{\bf filter} & \term{{\bf 1}}{\psi}{{\cal X}}
			& \multicolumn{2}{c|}{{\neg}{\psi} \;=\; {\cal X} \tensorimplytarget \psi = \looer{\psi}{{\cal X}}}
				\\ \hline
	 \end{array}$
\end{center}
\caption{{\bf Negations of relations} \label{negations}}
\end{table}
When source and target orders are the induced orders
${\cal X}_0^{\mu}$ and ${\cal X}_0^{\mu}$,
negation is
\begin{equation}
	\neg{\mu} 
	\;=\;	\mu \tensorimplysource {\cal X}_0^{\mu} 
	\;=\;	{\cal X}_1^{\mu} \tensorimplytarget \mu
	\;=\;	\mu \tensorimplysource \mu \tensorimplytarget \mu. \label{negrel}
\end{equation}
Let us recall the initial discussion in Section~\ref{ConceptualCollectives}
about the order-theoretic constructions
induced by Wille's concept lattice construction.
In addition to the original relation ({\bf source}-{\bf target}),
and the two induced orders ({\bf source}-{\bf source} and {\bf target}-{\bf target}),
the only other definable relation 
with this data is a relation opposite to $\mu$ ({\bf target}-{\bf source}),
a relation closed w.r.t.\ the induced orders
\begin{equation}
	\term{{\cal X}_1^{\mu}}{\overline{\mu}}{{\cal X}_0^{\mu}},
\end{equation}
which is defined
either by
\begin{center}
	$\begin{array}{r@{\;\;\;\mbox{iff}\;\;\;}l}
		x_1{\overline{\mu}}x_0
			& \hat{\imath}_1(x_1) \leq_{\rm CL} \hat{\imath}_0(x_0) \\
			& {\mu}x_1 \subseteq {\downarrow}x_0 \\
			& (\forall x'_0 \memberof X_0)\; x'_0{\mu}x_1 \;\mbox{implies}\; x'_0 \leq_0^\mu x_0 \\
			& x_1{(\mu \tensorimplysource {\cal X}_0^{\mu})}x_0
	 \end{array}$
\end{center}
or by
\begin{center}
	$\begin{array}{r@{\;\;\;\mbox{iff}\;\;\;}l}
		x_1{\overline{\mu}}x_0
			& \hat{\imath}_1(x_1) \leq_{\rm CL} \hat{\imath}_0(x_0) \\
			& x_0{\mu} \subseteq {\uparrow}x_1 \\
			& (\forall x'_1 \memberof X_1)\; x_0{\mu}x'_1 \;\mbox{implies}\; x_1 \leq_1^\mu x'_1 \\
			& x_1{({\cal X}_1^{\mu} \tensorimplytarget \mu)}x_0 .
	 \end{array}$
\end{center}
\begin{Proposition}
	Negation $\term{{\cal X}_1^\mu}{\neg{\mu}}{{\cal X}_0^\mu}$ is the opposite relation induced by the concept lattice
	\[ \neg{\mu} = \overline{\mu}. \]
\end{Proposition}
The negation of the original planetary relationship,
which is the opposite relation induced by the concept lattice of the planetary context,
is displayed in Table~\ref{negrelplanets}.
\scriptsize
\begin{table}
\begin{center}	\begin{tabular}{|l|@{\hspace{1.5mm}}c@{\hspace{1.5mm}}c@{\hspace{1.5mm}}c@{\hspace{1.5mm}}c@{\hspace{1.5mm}}c@{\hspace{1.5mm}}c@{\hspace{1.5mm}}c@{\hspace{1.5mm}}c@{\hspace{1.5mm}}c@{\hspace{0.5mm}}|} \hline
		   & Me & V & E & Ma & J & S & U & N & P  \\ \hline
		ss &&&&&&&&& \\
		sm &&&&&&& $\times$ & $\times$ & \\
		sl &&&&& $\times$ & $\times$ &&& \\
		dn &&&&&&&&& \\
		df &&&&&&&&& \\
		my &&&&&&&&& \\
		mn & $\times$ & $\times$ &&&&&&& \\ \hline
	\end{tabular}
\end{center}
\caption{{\bf Negation relation for the planetary context} \label{negrelplanets}}
\end{table}
\normalsize
Intuitively,
the negation relation $\neg{\mu}$ specifies implicational information.
From one viewpoint,
a negation relationship
$x_1{ \neg{\mu} }x_0$ iff $x_1{(\mu \tensorimplysource {\cal X}_0^{\mu})}x_0$
exists between individual attribute $x_1$ and individual entity $x_0$
when any individual entities having individual attribute $x_1$ also have collective attribute $x_0$.
Since $x_1$ represents its collection of individual entities as a collective entity,
we can say that $x_1{ \neg{\mu} }x_0$ when collective entity $x_1$ has collective attribute $x_0$.
Arguing from the dual position,
a negation relationship
$x_1{ \neg{\mu} }x_0$ iff $x_1{({\cal X}_1^{\mu} \tensorimplytarget \mu)}x_0$
exists between individual attribute $x_1$ and individual entity $x_0$
when any individual attributes of individual entity $x_0$ are also attributes of collective entity $x_1$.
Since $x_0$ represents its collection of individual attributes as a collective attribute,
we can say that $x_1{ \neg{\mu} }x_0$ when $x_0$ is a collective attribute of collective entity $x_1$
(the same meaning as before).


\section{Contextual Fibration}

Now we are at a crucial point in our analysis of concept construction.
The resolution of the above problem requires a change of viewpoint --- a shift in our conceptual framework.
In order to make visible and explicit the hidden relationship of a context,
we must define contexts as follows.
A {\em formal context\/}\footnote{The complexity of this definition,
when compared with Wille's original set-theoretic definition,
is quite noticeable
--- but the main argument of this paper is that this complexity is unavoidable,
if we are committed to the philosophical principles that
(1) intensions should be represented explicitly, and
(2) conceptual structures are coherent units of thought
consisting of an extensional aspect and an intensional aspect
which determine each other
by means of the notion of commonality or sharing.}
${\cal X} = \quadruple{{\cal X}_0}{\mu_{01}}{\mu_{10}}{{\cal X}_1}$
is a pair of orders ${\cal X}_0 = \pair{X_0}{\leq_0}$ and ${\cal X}_1 = \pair{X_1}{\leq_1}$,
and a pair of oppositely directed closed relations
$\term{{\cal X}_0}{{\mu_{01}}}{{\cal X}_1}$
and
$\term{{\cal X}_1}{{\mu_{10}}}{{\cal X}_0}$
between them,
which satisfy the orthogonal constraints
$\mu_{01} \circ \mu_{10} \preceq {\cal X}_0$ and $\mu_{10} \circ \mu_{01} \preceq {\cal X}_1$.
The four components of a formal context can be arranged in a matrix
\[ {\cal X}
	=
	\left( \begin{array}{cc} {\cal X}_0 & \mu_{01} \\ \mu_{10} & {\cal X}_1 \end{array} \right). \]
Formal contexts in this sense are quite general.
Let $\pair{{\cal X}}{t}$ be any pair consisting of
an order ${\cal X} = \pair{X}{\leq_{\cal X}}$
and
a monotonic function $\morphism{{\cal X}}{t}{\overline{\bf 2}}$ from ${\cal X}$ to the binary order $\overline{\bf 2}$
defined by
$\overline{\bf 2} \define \pair{ \{ 0,1 \} }{\vdashv}$,
where $0 \vdashv 1$ and $1 \vdashv 0$.
This pair specifies a formal context
\[ {\cal X}^t = \left( \begin{array}{cc} {\cal X}_0 & \mu_{01} \\ \mu_{10} & {\cal X}_1 \end{array} \right) \]
called the {\em $t$-partition\/} of ${\cal X}$,
whose components are defined by
\begin{center}
	$\begin{array}{c@{\hspace{3mm}}c@{\hspace{3mm}}l}
		\term{{\cal X}_0}{{\leq_0}}{{\cal X}_0}
			& \define
				& t^{\rm -1}(0 \vdashv 0) = \{ (x'_0,x_0) \mid x'_0,x_0 \memberof X_0, x'_0 \leq_X x_0 \} \\
		\term{{\cal X}_1}{{\leq_1}}{{\cal X}_1}
			& \define
				& t^{\rm -1}(1 \vdashv 1) = \{ (x_1,x'_1) \mid x_1,x'_1 \memberof X_1, x_1 \leq_X x'_1 \} \\
		\term{{\cal X}_0}{\mu_{01}}{{\cal X}_1}
			& \define
				& t^{\rm -1}(0 \vdashv 1) = \{ (x_0,x_1) \mid x_0 \memberof X_0, x_1 \memberof X_1, x_0 \leq_X x_1 \} \\
		\term{{\cal X}_0}{\mu_{10}}{{\cal X}_1}
			& \define
				& t^{\rm -1}(1 \vdashv 0) = \{ (x_1,x_0) \mid x_1 \memberof X_1, x_0 \memberof X_0, x_1 \leq_X x_0 \}.
	 \end{array}$
\end{center}
The monotonic function $t$,
called a {\em tag\/} or {\em index\/} function,
indicates from which component order an element in the sum originates $t(x_0) = 0$ and $t(x_1) = 1$,
and functions as a partition or fibration $\{ X_0, X_1 \}$ of the underlying set $X = X_0 + X_1$.

Sum orders define contextual fibrations:
given any two orders ${\cal X}_0$ and ${\cal X}_1$,
the disjoint union order ${\cal X}_0 + {\cal X}_1$
and the linear sum order ${\cal X}_0 \oplus {\cal X}_1$ define the contexts
\begin{center}
\begin{tabular}{c@{\hspace{1cm}}c@{\hspace{1cm}}c}
	$\left( \begin{array}{cc} {\cal X}_0 & \bot \\ \bot & {\cal X}_1 \end{array} \right)$
	& and &
	$\left( \begin{array}{cc} {\cal X}_0 & \top \\ \bot & {\cal X}_1 \end{array} \right)$,
\end{tabular}
\end{center}
respectively.
Product orders define contextual fibrations:
given any order ${\cal X}$,
the binary product projection $\longmorphism{\product{\overline{\bf 2}}{{\cal X}}}{p_{\overline{2}}}{\overline{\bf 2}}$
and the Boolean product projection $\longmorphism{\product{{\bf 2}}{{\cal X}}}{p_{2}}{{\bf 2}}$
define the contexts
\begin{center}
\begin{tabular}{c@{\hspace{1cm}}c@{\hspace{1cm}}c}
	$\left( \begin{array}{cc} {\cal X} & {\cal X} \\ {\cal X} & {\cal X} \end{array} \right)$
	& and &
	$\left( \begin{array}{cc} {\cal X} & {\cal X} \\ \bot & {\cal X} \end{array} \right)$,
\end{tabular}
\end{center}
respectively,
with off-diagonal entries being the identity relation.

Of special significance in the analysis of concept construction is the {\em contextual closure\/} of a binary relation $\term{X_0}{\mu}{X_1}$.
This is the formal context
\[ {\cal X}^{\mu} 
   = \left( \begin{array}{cc} {\cal X}_0^{\mu} & \mu \\ {\neg}{\mu} & {\cal X}_1^{\mu} \end{array} \right)
   = \left( \begin{array}{cc} \mu \tensorimplytarget \mu & \mu \\ \mu \tensorimplysource \mu \tensorimplytarget \mu & \mu \tensorimplysource \mu \end{array} \right)
 \]
consisting of the induced orders
${\cal X}_0^{\mu} \define \mu \tensorimplytarget \mu$
and
${\cal X}_1^{\mu} \define \mu \tensorimplysource \mu$,
the given relation $\mu$ which is closed w.r.t.\ these orders,
and the negation relation ${\neg}{\mu}$.
The contextual closure is the largest formal context containing $\mu$.
Contextual closure transforms the combinatorial object $\triple{X_0}{X_1}{\mu}$,
the original notion of formal context of Wille,
into the algebraic and potentially abstract object ${\cal X}^\mu$
--- thereby tremendously increasing the power, flexibility and expressibility
of the basic mathematical object under study.
The central result of this paper called the Equivalence Theorem,
demonstrates the equivalence of concept lattice construction and Dedekind-MacNeille completion by using contextual closure.

Formal contexts can be compared.
For any two contexts
\[
	{\cal Y} = \left( \begin{array}{cc} {\cal Y}_0 & \nu_{01} \\ \nu_{10} & {\cal Y}_1 \end{array} \right)
	\hspace{5mm}\mbox{and}\hspace{5mm}
	{\cal X} = \left( \begin{array}{cc} {\cal X}_0 & \mu_{01} \\ \mu_{10} & {\cal X}_1 \end{array} \right),
\]
a {\em map of formal contexts\/} $\Morphism{{\cal Y}}{f}{{\cal X}}$
from context ${\cal Y}$ to context ${\cal X}$
is a pair $f \define \pair{f_0}{f_1}$ of monotonic functions between component orders
$\morphism{{\cal Y}_0}{f_0}{{\cal X}_0}$ and $\morphism{{\cal Y}_1}{f_1}{{\cal X}_1}$,
which satisfy both of the symbolic conditions
\begin{center}
\begin{tabular}{c@{\hspace{3cm}}c}
	$\begin{array}{r@{\hspace{0.1mm}}c@{\hspace{2mm}}c@{\hspace{2mm}}c@{\hspace{0.1mm}}l}
			& {\cal Y}_0
				& \stackrel{\nu_{01}}{\rightharpoondown}
					& {\cal Y}_1
						& \\
		\mbox{\footnotesize$f_0^{\triangleright}$\normalsize}
			& \downarrow
				& \mbox{\raisebox{1ex}{$\succeq$}}
					& \downarrow
						& \mbox{\footnotesize$f_1^{\triangleright}$\normalsize} \\
			& {\cal X}_0
				& \stackrel{\mu_{10}}{\rightharpoondown}
					& {\cal X}_1
						&
	 \end{array}$
	&
	$\begin{array}{r@{\hspace{0.1mm}}c@{\hspace{2mm}}c@{\hspace{2mm}}c@{\hspace{0.1mm}}l}
			& {\cal Y}_0
				& \stackrel{\nu_{10}}{\leftharpoondown}
					& {\cal Y}_1
						& \\
		\mbox{\footnotesize$f_0^{\triangleright}$\normalsize}
			& \downarrow
				& \mbox{\raisebox{1ex}{$\preceq$}}
					& \downarrow
						& \mbox{\footnotesize$f_1^{\triangleright}$\normalsize} \\
			& {\cal X}_0
				& \stackrel{\mu_{10}}{\leftharpoondown}
					& {\cal X}_1
						&
	 \end{array}$
\end{tabular}
\end{center}
expressed formally as two sets of equvialent conditions
\begin{center}
	\begin{tabular}{c@{\hspace{3cm}}c}
	$\begin{array}{r@{\;\preceq\;}l}
		f_0^{\triangleleft} \circ \nu_{01} \circ f_1^{\triangleright} & \mu_{01} \\
		\nu_{01} \circ f_1^{\triangleright} & f_0^{\triangleright} \circ \mu_{01} \\
		\nu_{01} & f_0^{\triangleright} \circ \mu_{01} \circ f_1^{\triangleleft}
	 \end{array}$
	&
	$\begin{array}{r@{\;\preceq\;}l}
		f_1^{\triangleleft} \circ \nu_{10} \circ f_0^{\triangleright} & \mu_{10} \\
		\nu_{10} \circ f_0^{\triangleright} & f_1^{\triangleright} \circ \mu_{10} \\
		\nu_{10} & f_1^{\triangleright} \circ \mu_{10} \circ f_0^{\triangleleft}
	 \end{array}$
	\end{tabular}
\end{center}
Maps of formal contexts preserve the {\em has\/} relationships:
if $y_0$ has attribute $y_1$ w.r.t.\ the relationship $\term{{\cal Y}_0}{\nu_{01}}{{\cal Y}_1}$,
symbolically $y_0{\nu_{01}}y_1$,
then $f_0(x_0)$ has attribute $f_1(x_1)$ w.r.t.\ the relationship $\term{{\cal X}_0}{\mu_{01}}{{\cal X}_1}$,
symbolically $f_0(y_0){\mu_{01}}f_1(y_1)$.
Similarly for the opposite direction.

Formal contexts and their maps form the category
${\bf Cxt}$.
There is a fully-faithful embedding functor
called {\em inclusion-of-identity-contexts\/}
\[ {\bf Ord} \stackrel{\rm Inc}{\longrightarrow} {\bf Cxt} \]
which maps orders ${\cal X}$ to identity contexts
\[ {\rm Inc}({\cal X})
   \hspace{5mm} \define \hspace{5mm}
   \left( \begin{array}{cc} {\cal X} & {\cal X} \\ {\cal X} & {\cal X} \end{array} \right) \]
and maps monotonic functions $\morphism{{\cal Y}}{f}{{\cal X}}$
to their doubles
${\rm Inc}({\cal Y}) \stackrel{{\rm Inc}(f)}{\Longrightarrow} {\rm Inc}({\cal X})$,
defined as the pair ${\rm Inc}(f) \define \pair{f}{f}$.

\begin{enumerate}
	\item {\bf Opposite:} The opposite involution,
		taking contexts
		\[ {\cal X} = \left( \begin{array}{cc} {\cal X}_0 & \mu_{01} \\ \mu_{10} & {\cal X}_1 \end{array} \right) \]
		to their opposites
		\[ {\cal X}^{\rm op}
			 = \left( \begin{array}{cc} {\cal X}_1^{\rm op} & \mu_{10}^{\rm op} \\ \mu_{01}^{\rm op} & {\cal X}_0^{\rm op} \end{array} \right), \]
		interchanges individual attributes with individual entities
		and
		interchanges collective attributes with collective entities.
	\item {\bf Terminal:} There is a very simple formal context
		\[ {\bf 1} = \left( \begin{array}{cc} {\bf 1} & {\bf 1} \\ {\bf 1} & {\bf 1} \end{array} \right) \]
		having exactly one entity and one related attribute.
		It consists of two copies of the unit order
		${\cal X}_0 = {\bf 1} = {\cal X}_1$,
		plus two copies of the identity relation
		$\mu_{01} = {\bf 1} = \mu_{10}$.
		This context is terminal,
		since from any formal context ${\cal X}$
		there is a unique map of contexts $\Morphism{{\cal X}}{\top_{\cal X}}{{\bf 1}}$
		to the {\em terminal context\/} {\bf 1},
		consisting of the pair of unique monotonic functions to the unit order
		$\top_{\cal X} = \pair{\top_{{\cal X}_0}}{\top_{{\cal X}_1}}$.
	\item {\bf Inverse image:}
		For any contexts
		\[ {\cal X} = \left( \begin{array}{cc} {\cal X}_0 & \mu_{01} \\ \mu_{10} & {\cal X}_1 \end{array} \right) \]
		and any pair $f = \pair{f_0}{f_1}$ of monotonic functions
		$\morphism{{\cal Y}_0}{f_0}{{\cal X}_0}$ and $\morphism{{\cal Y}_1}{f_1}{{\cal X}_1}$,
		there is an inverse image context
		\[ f^{\rm -1}({\cal X})
		 = \left( \begin{array}{cc} {\cal Y}_0 & f_0^\triangleright \circ \mu_{01} \circ f_1^\triangleleft \\
		                            f_1^\triangleright \circ \mu_{10} \circ f_0^\triangleleft & {\cal Y}_1 \end{array} \right), \]
		with canonical map of formal contexts
		$\Morphism{f^{\rm -1}({\cal X})}{f}{{\cal X}}$.
	\item {\bf Product:} Given any pair of formal contexts
		\[
			{\cal Y} = \left( \begin{array}{cc} {\cal Y}_0 & \nu_{01} \\ \nu_{10} & {\cal Y}_1 \end{array} \right)
				\hspace{5mm}\mbox{and}\hspace{5mm}
			{\cal X} = \left( \begin{array}{cc} {\cal X}_0 & \mu_{01} \\ \mu_{10} & {\cal X}_1 \end{array} \right) 
		\]
		the context
		\[ \product{{\cal Y}}{{\cal X}}
			=
		   \left( \begin{array}{cc} \product{{\cal Y}_0}{{\cal X}_0} & \product{{\nu}_{01}}{{\mu}_{01}} \\
		                            \product{{\nu}_{10}}{{\mu}_{10}} & \product{{\cal Y}_1}{{\cal X}_1} \end{array} \right) \]
		is the product context.
	\item {\bf Meet:} Given any pair of formal contexts
		\[
			{\cal X}' = \left( \begin{array}{cc} {\cal X}'_0 & \mu'_{01} \\ \mu'_{10} & {\cal X}'_1 \end{array} \right)
				\hspace{5mm}\mbox{and}\hspace{5mm}
			{\cal X} = \left( \begin{array}{cc} {\cal X}_0 & \mu_{01} \\ \mu_{10} & {\cal X}_1 \end{array} \right) 
		\]
		over the same pair of underlying sets $X_0$ and $X_1$,
		the context
		\[ {\cal X}' \wedge {\cal X}
			=
		   \left( \begin{array}{cc} {\cal X}'_0 \wedge {\cal X}_0 & {\mu}'_{01} \wedge {\mu}_{01} \\
		                            {\mu}'_{10} \wedge {\mu}_{10} & {\cal X}'_1 \wedge {\cal X}_1 \end{array} \right) \]
		is the meet context over the same pair.
\end{enumerate}

\begin{Proposition}
	The category of formal contexts ${\bf Cxt}$ is complete;
	limits exist for all diagrams.
	It is also involutionary and fibered.
\end{Proposition}

Let us take stock of our current situation.
In order to define derivation in a coherent fashion in concept construction,
closely following the philosophy that a concept consists of an extent and an intent that determine each other,
and respecting any and all relationships that are actually present,
we have been forced to change our starting framework
--- our notion of a formal context.
It seems appropriate that we should start over in our analysis of concept construction.
In the definition of a concept \cite{Wille}
the extent and the intent determine each other.
This means that in concept construction we must end where we started
--- since we have ended with a pair of order ideals,
instead of starting with just a single order ideal of entities $\phi \in \powerof{{\cal X}_0^{\rm op}}$
we should have started with a pair $(\phi_0,\phi_1)$,
an order ideal of individual entities $\phi_0 \in \powerof{{\cal X}_0^{\rm op}}$
and an order ideal of collective entities $\phi_1 \in \powerof{{\cal X}_1^{\rm op}}$,
subject to the ideal assertions
${\mu_{01}} \circ \phi_1 \preceq \phi_0$
and
${\mu_{10}} \circ \phi_0 \preceq \phi_1$.
A formal concept will consist of a quadruple
$\pair{[\phi_0,\phi_1]}{(\psi_0,\psi_1)}$,
a pair of extent-intent pairs
with extent and intent described by
\begin{center}
	$\begin{array}{r@{\hspace{2mm}}lp{9cm}}
		{\bf extent}
			& \left\{
				\begin{array}{ll}
					{\bf individual}
						& \phi_0 \memberof \powerof{{\cal X}_0^{\rm op}} \\
					{\bf collective}
						& \phi_1 \memberof \powerof{{\cal X}_1^{\rm op}}
				\end{array}
			 \right\}
				& satisfying ideal assertions
			  $\left\{
				\begin{array}{l}
					\mu_{01} \circ \phi_1 \;\preceq\; \phi_0 \\
					\mu_{10} \circ \phi_0 \;\preceq\; \phi_1
				   \end{array}
			  \right.$, \\ \\
		{\bf intent}
			& \left\{
				\begin{array}{ll}
					{\bf collective}
						& \psi_0 \memberof \powerof{{\cal X}_0} \\
					{\bf individual}
						& \psi_1 \memberof \powerof{{\cal X}_1}
				\end{array}
			 \right\}
				& satisfying filter assertions
			  $\left\{
				\begin{array}{l}
					\psi_0 \circ \mu_{01} \;\preceq\; \psi_1 \\
					\psi_1 \circ \mu_{10} \;\preceq\; \psi_0
				   \end{array}
			   \right.$.
	 \end{array}$
\end{center}
The direct and inverse phases of derivation each consist of four operations.
The component operators for derivation data flow are described as follows.
\begin{itemize}
	\item {\bf direct derivation}
			$\left( \begin{array}{cc} \upper{(\;)}{{\cal X}_0} & \derivedir{(\;)}{\mu_{01}} \\ \derivedir{(\;)}{\mu_{10}} & \upper{(\;)}{{\cal X}_1} \end{array} \right)$
			$\left\{\mbox{\begin{tabular}{c@{\hspace{2mm}}l}
				$\upper{(\;)}{{\cal X}_0}$
					& collective intent of individual extent \\
				$\derivedir{(\;)}{\mu_{10}}$
					& collective intent of collective extent \\
				$\derivedir{(\;)}{\mu_{01}}$
					& individual intent of individual extent \\
				$\upper{(\;)}{{\cal X}_1}$
					& individual intent of collective extent
			\end{tabular}}\right.$
	\item {\bf inverse derivation}
			$\left( \begin{array}{cc} \looer{(\;)}{{\cal X}_0} & \deriveinv{(\;)}{\mu_{01}} \\ \deriveinv{(\;)}{\mu_{10}} & \looer{(\;)}{{\cal X}_1} \end{array} \right)$
			$\left\{\mbox{\begin{tabular}{c@{\hspace{2mm}}l}
				$\looer{(\;)}{{\cal X}_0}$
					& individual extent of collective intent \\
				$\deriveinv{(\;)}{\mu_{01}}$
					& individual extent of individual intent \\
				$\deriveinv{(\;)}{\mu_{10}}$
					& collective extent of collective intent \\
				$\looer{(\;)}{{\cal X}_1}$
					& collective extent of individual intent.
			\end{tabular}}\right.$
\end{itemize}
The data flow in the two phases of concept construction is illustrated in Figure~\ref{dataflow}.
\begin{figure}
\begin{center}
	\begin{picture}(180,80)(-5,-40)
		\put(2,32){\makebox(0,0){\small$\powerof{{\cal X}_0^{\rm op}}$\normalsize}}
		\put(2,-28){\makebox(0,0){\small$\powerof{{\cal X}_1^{\rm op}}$\normalsize}}
		\put(82,32){\makebox(0,0){\small${(\powerof{{\cal X}_0})}^{\rm op}$\normalsize}}
		\put(82,-28){\makebox(0,0){\small${(\powerof{{\cal X}_1})}^{\rm op}$\normalsize}}
		\put(162,32){\makebox(0,0){\small$\powerof{{\cal X}_0^{\rm op}}$\normalsize}}
		\put(162,-28){\makebox(0,0){\small$\powerof{{\cal X}_1^{\rm op}}$\normalsize}}
		\put(45,10){\makebox(0,0){\footnotesize$\derivedir{(\,)}{\mu_{01}}$\normalsize}}
		\put(45,-10){\makebox(0,0){\footnotesize$\derivedir{(\,)}{\mu_{10}}$\normalsize}}
		\put(123,10){\makebox(0,0){\footnotesize$\deriveinv{(\,)}{\mu_{10}}$\normalsize}}
		\put(123,-10){\makebox(0,0){\footnotesize$\deriveinv{(\,)}{\mu_{01}}$\normalsize}}
		\put(40,40){\makebox(0,0){\footnotesize$\upper{(\,)}{{\cal X}_0}$\normalsize}}
		\put(40,-40){\makebox(0,0){\footnotesize$\upper{(\,)}{{\cal X}_1}$\normalsize}}
		\put(124,40){\makebox(0,0){\footnotesize$\looer{(\,)}{{\cal X}_0}$\normalsize}}
		\put(124,-40){\makebox(0,0){\footnotesize$\looer{(\,)}{{\cal X}_1}$\normalsize}}
		\put(-10,0){\makebox(0,0){\footnotesize$\idealexists{\mu_{01}}{}$\normalsize}}
		\put(68,0){\makebox(0,0){\footnotesize${\filterforall{\mu_{01}}{}}^{\rm op}$\normalsize}}
		\put(95,0){\makebox(0,0){\footnotesize${\filterexists{\mu_{01}}{}}^{\rm op}$\normalsize}}
		\put(172,0){\makebox(0,0){\footnotesize$\idealforall{\mu_{01}}{}$\normalsize}}
		\put(0,-20){\vector(0,1){40}}
		\put(83,20){\vector(0,-1){40}}
		\put(77,-20){\vector(0,1){40}}
		\put(160,20){\vector(0,-1){40}}
		\put(15,30){\vector(1,0){50}}
		\put(15,-30){\vector(1,0){50}}
		\put(95,30){\vector(1,0){50}}
		\put(95,-30){\vector(1,0){50}}
		\put(13,21){\vector(4,-3){55}}
		\put(13,-21){\vector(4,3){55}}
		\put(93,-21){\vector(4,3){55}}
		\put(93,21){\vector(4,-3){55}}
		\put(-23,-40){\dashbox{1}(46,80){}}
		\put(56,-40){\dashbox{1}(50,80){}}
		\put(137,-40){\dashbox{1}(46,80){}}
	\end{picture}
\end{center}
\caption{{\bf Data flow in derivation} \label{dataflow}}
\end{figure}
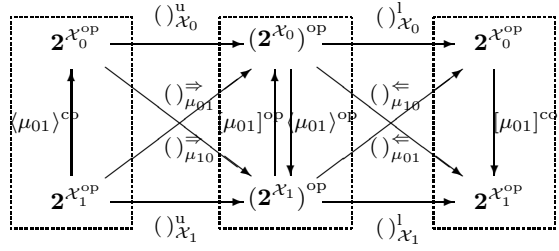
The requirement that conceptual extent and conceptual intent determine each other
is expressed by the constraining definitions
\begin{center}
	\begin{tabular}{c@{\hspace{1cm}}c}
		$\begin{array}{r@{\hspace{2mm}}l}
			{\bf extent}
				& \left\{
					\begin{array}{r@{\;\define\;}l}
						\phi_0 & \looer{(\psi_0)}{{\cal X}_0} \wedge \deriveinv{(\psi_1)}{\mu_{01}} \\
						\phi_1 & \deriveinv{(\psi_0)}{\mu_{10}} \wedge \looer{(\psi_1)}{{\cal X}_1}
					 \end{array}
				  \right.
		 \end{array}$
		&
		$\begin{array}{r@{\hspace{2mm}}l}
			{\bf intent}
				& \left\{
					\begin{array}{r@{\;\define\;}l}
						\psi_0 & \upper{(\phi_0)}{{\cal X}_0} \wedge \derivedir{(\phi_1)}{\mu_{10}} \\
						\psi_1 & \derivedir{(\phi_0)}{\mu_{01}} \wedge \upper{(\phi_1)}{{\cal X}_1}.
					 \end{array}
				  \right.
		 \end{array}$
	\end{tabular}
\end{center}


\section{Contextual Summation}

A coherent approach to concept construction
is the definition of an appropriate notion of a sum of a formal context
\[ {\cal X}
	=
	\left( \begin{array}{cc} {\cal X}_0 & \mu_{01} \\ \mu_{10} & {\cal X}_1 \end{array} \right). \]
The formal context ${\cal X}$
(pair of orthogonal relations
 $\term{{\cal X}_0}{{\mu_{01}}}{{\cal X}_1}$
 and
 $\term{{\cal X}_1}{{\mu_{10}}}{{\cal X}_0}$)
specifies external constraints
between the component orders ${\cal X}_0$ and ${\cal X}_1$ in two senses:
either a source constraint or a target constraint.
Two relations
$\term{{\cal X}_0}{\rho_0}{{\cal Y}}$ and $\term{{\cal X}_1}{\rho_1}{{\cal Y}}$
from the component source orders ${\cal X}_0$ and ${\cal X}_1$
to a common target order ${\cal Y}$
satisfy the external source constraints specified by the formal context when
\begin{equation}
	{\mu}_{01} \circ \rho_1 \preceq \rho_0
		\hspace{1cm} \mbox{and} \hspace{1cm}
	{\mu}_{10} \circ \rho_0 \preceq \rho_1. \label{source}
\end{equation}
When ${\cal Y} = {\bf 1}$ the two relations are order ideals
$\rho_0 \memberof \powerof{{\cal X}_0^{\rm op}}$
and
$\rho_1 \memberof \powerof{{\cal X}_1^{\rm op}}$,
which satisfy the ideal assertions~\ref{source}.
Two relations
$\term{{\cal W}}{\sigma_0}{{\cal X}_0}$ and $\term{{\cal W}}{\sigma_1}{{\cal X}_1}$
from a common source order ${\cal W}$
to the component target orders ${\cal X}_0$ and ${\cal X}_1$
satisfy the external target constraint specified by the formal context when
\begin{equation}
	\sigma_0 \circ {\mu}_{01} \preceq \sigma_1
		\hspace{1cm} \mbox{and} \hspace{1cm}
	\sigma_1 \circ {\mu}_{10} \preceq \sigma_0. \label{target}
\end{equation}
When ${\cal W} = {\bf 1}$ the two relations are order filters
$\sigma_0 \memberof \powerof{{\cal X}_0}$
and
$\sigma_1 \memberof \powerof{{\cal X}_1}$,
which satisfy the filter assertions~\ref{target}.

Any formal context ${\cal X}$,
which specifies such collections of external constraints,
can be internalized or centralized
as a {\em sum\/} order ${\oplus}{\cal X} = \pair{X}{\leq_X}$
consisting of the disjoint union of elements $X = X_0 {+} X_1$
with order relation $\leq_X$ defined by
\begin{center}
	\begin{tabular}{r@{\hspace{3mm}}c@{\hspace{3mm}}l}
		$x'_0 \leq_X x_0$    & iff     & $x'_0 \leq_0 x_0$	\\
		$x_0 \leq_X x_1$     & iff     & $x_0{\mu_{01}}x_1$	\\
		$x_1 \leq_X x_0$     & iff     & $x_1{\mu_{10}}x_0$	\\
		$x_1 \leq_X x'_1$    & iff     & $x_1 \leq_1 x'_1$
	\end{tabular}
\end{center}
for all elements $x_0 \in X_0$ and $x_1 \in X_1$.
The coproduct injections for the underlying disjoint union are monotonic functions
${\cal X}_0 \stackrel{i_0}{\longrightarrow} {\oplus}{\cal X}
            \stackrel{i_1}{\longleftarrow}  {\cal X}_1$,
which satisfy the defining conditions
\begin{center}
	\begin{tabular}{rl}
		$\left. \begin{array}{r@{\;=\;}lcr@{\;=\;}l}
			i_0^{\triangleright} \circ i_0^{\triangleleft} & {\cal X}_0 && i_0^{\triangleright} \circ i_1^{\triangleleft} & {\mu}_{01} \\
			i_1^{\triangleright} \circ i_0^{\triangleleft} & {\mu}_{10} && i_1^{\triangleright} \circ i_1^{\triangleleft} & {\cal X}_1 \\
		 \end{array} \right\}$
		& ``suborder disjointness equations''
		\\
		$\left. \begin{array}{r@{\;=\;}l}
			(i_0^{\triangleleft} \circ i_0^{\triangleright})
			 \vee
			(i_1^{\triangleleft} \circ i_1^{\triangleright})
				& {\oplus}{\cal X}
		 \end{array} \right\}$
		& ``covering equation''
	\end{tabular}
\end{center}
The sum of the terminal context
\[ {\bf 1} = \left( \begin{array}{cc} {\bf 1} & {\bf 1} \\ {\bf 1} & {\bf 1} \end{array} \right) \]
is the binary order $\overline{\bf 2} = {\oplus}{\bf 1}$ .

Given any pair of relations
${\cal X}_0 \stackrel{\rho_0}{\rightharpoondown} {\cal Y} \stackrel{\rho_1}{\leftharpoondown} {\cal X}_1$
which satisfy the external source contraints
${\mu}_{01} \circ \rho_1 \preceq \rho_0$ and ${\mu}_{10} \circ \rho_0 \preceq \rho_1$
specified by the formal context,
there is a unique ``mediating'' relation
${\oplus}{\cal X} \stackrel{\rho}{\rightharpoondown} {\cal Y}$,
symbolized by $\rho = \relcopair{\rho_0}{\rho_1}{}{}$
and called the relative copairing of $\rho_0$ and $\rho_1$,
which satisfies the rules
\begin{center}
	$\begin{array}{r@{\;=\;}l}
        i_0^{\triangleright} \circ \relcopair{\rho_0}{\rho_1}{}{} & \rho_0 \\
        i_1^{\triangleright} \circ \relcopair{\rho_0}{\rho_1}{}{} & \rho_1
	 \end{array}$
\end{center}
Just define
$\relcopair{\rho_0}{\rho_1}{}{}
= (i_0^{\triangleleft} \circ \rho_0) \vee (i_1^{\triangleleft} \circ \rho_1)$.
These properties say that the sum order ${\oplus}{\cal X}$
is a coproduct {\em relative\/} to the external constraints specified by the formal context.
Clearly, the copairing operator $\relcopair{\;}{\;}{}{}$ is an order-isomorphism:
$\rho_0 \preceq \rho'_0$ and $\rho_1 \preceq \rho'_1)$ implies
$\relcopair{\rho_0}{\rho_1}{}{} \preceq \relcopair{\rho'_0}{\rho'_1}{}{}$.
Then
$\relcopair{\rho_0}{\rho_1}{}{} \wedge \relcopair{\rho'_0}{\rho'_1}{}{}
 = \relcopair{\rho_0 \wedge \rho'_0}{\rho_1 \wedge \rho'_1}{}{}$.
For any relation $\term{{\cal Y}}{\sigma}{{\cal Z}}$ it is immediate that
$\relcopair{\rho_0}{\rho_1}{}{} \circ \sigma
 = \relcopair{\rho_0 \circ \sigma}{\rho_1 \circ \sigma}{}{}$.
An alternate definition of copairing,
in terms of implications instead of product,
is given by
$\relcopair{\rho_0}{\rho_1}{}{}
= (i_0^{\triangleright} \tensorimplysource \rho_0) \wedge (i_1^{\triangleright} \tensorimplysource \rho_1)$.
The ``overlap'' of the $\rho_0$-part and the $\rho_1$-part of the source pairing is the relation
$i_0^{\triangleleft} \circ \mu \circ \rho_1
= (i_0^{\triangleleft} \circ \rho_0) \wedge (i_1^{\triangleleft} \circ \rho_1)$.
When ${\cal Y} = {\bf 1}$
the copairing operator defines an isomorphism
\[ \powerof{{\cal X}_0^{\rm op}} {\times}_{\mu} \powerof{{\cal X}_1^{\rm op}}
   \stackrel{\relcopair{}{}{}{}}{\cong}
   \powerof{ {{\oplus}{\cal X}}^{\rm op} }. \]

Dually, given any pair of relations
${\cal X}_0 \stackrel{\sigma_0}{\leftharpoondown} {\cal W} \stackrel{\sigma_1}{\rightharpoondown} {\cal X}_1$
which satisfy the external target contraints
$\sigma_0 \circ {\mu}_{01} \preceq \sigma_1$ and $\sigma_1 \circ {\mu}_{10} \preceq \sigma_0$
specified by the formal context,
there is a unique ``mediating'' relation
${\cal W} \stackrel{\sigma}{\rightharpoondown} {\oplus}{\cal X}$,
symbolized by $\sigma = \relpair{\sigma_0}{\sigma_1}{}{}$
and called the relative pairing of $\sigma_0$ and $\sigma_1$,
which satisfies the rules
\begin{center}
	$\begin{array}{r@{\;=\;}l}
       \relpair{\sigma_0}{\sigma_1}{}{} \circ i_0^{\triangleleft} & \sigma_0 \\
       \relpair{\sigma_0}{\sigma_1}{}{} \circ i_1^{\triangleleft} & \sigma_1
	 \end{array}$
\end{center}
Just define
$\relpair{\sigma_0}{\sigma_1}{}{}
= (\sigma_0 \circ i_0^{\triangleright}) \vee (\sigma_1 \circ i_1^{\triangleright})$.
These properties say that the sum ${\oplus}{\cal X}$
is a product {\em relative\/} to the external constraints specified by the formal context.
Clearly, the pairing operator $\relpair{\;}{\;}{}{}$ is an order-embedding
and preserves meets.
For any relation $\term{{\cal V}}{\rho}{{\cal W}}$ it is immediate that
$\rho \circ \relpair{\sigma_0}{\sigma_1}{}{}
 = \relpair{\rho \circ \sigma_0}{\rho \circ \sigma_1}{}{}$.
An alternate definition of pairing,
in terms of implications instead of product,
is given by
$\relpair{\sigma_0}{\sigma_1}{}{}
= (\sigma_0 \tensorimplytarget i_0^{\triangleleft}) \wedge (\sigma_1 \tensorimplytarget i_1^{\triangleleft})$.
When ${\cal W} = {\bf 1}$
the pairing operator defines an isomorphism
\[ \powerof{{\cal X}_0} {\times}_{\mu} \powerof{{\cal X}_1}
   \stackrel{\relpair{}{}{}{}}{\cong}
   \powerof{ {\oplus}{\cal X} }. \]

There is a canonical monotonic function
\[ \morphism{{\oplus}{\cal X}}{\tau}{{\oplus}{\bf 1}} \]
from the sum order ${\oplus}{\cal X}$ to the binary order.
This canonical function,
called the {\em tag\/} or {\em index\/} function,
indicates from which component order an element in the sum originates:
$\tau(x_0) = 0$ and $\tau(x_1) = 1$,
and functions as a partition (fibration) of the underlying set $X$.
The components of the distributed context,
which are used in the summation (centralization) process,
are recoverable by the definitions
\begin{center}
	$\begin{array}{c@{\hspace{3mm}}c@{\hspace{3mm}}l}
		X_0
			& \define
				& \tau^{\rm -1}(0) = \{ x \mid x \memberof X, \tau(x) = 0 \} \\
		X_1
			& \define
				& \tau^{\rm -1}(1) = \{ x \mid x \memberof X, \tau(x) = 1 \} \\
		\term{{\cal X}_0}{{\leq_0}}{{\cal X}_0}
			& \define
				& \tau^{\rm -1}(0 \vdashv 0) = \{ (x'_0,x_0) \mid x'_0,x_0 \memberof X_0, x'_0 \leq_X x_0 \} \\
		\term{{\cal X}_1}{{\leq_1}}{{\cal X}_1}
			& \define
				& \tau^{\rm -1}(1 \vdashv 1) = \{ (x_1,x'_1) \mid x_1,x'_1 \memberof X_1, x_1 \leq_X x'_1 \} \\
		\term{{\cal X}_0}{\mu_{01}}{{\cal X}_1}
			& \define
				& \tau^{\rm -1}(0 \vdashv 1) = \{ (x_0,x_1) \mid x_0 \memberof X_0, x_1 \memberof X_1, x_0 \leq_X x_1 \} \\
		\term{{\cal X}_0}{\mu_{10}}{{\cal X}_1}
			& \define
				& \tau^{\rm -1}(1 \vdashv 0) = \{ (x_1,x_0) \mid x_1 \memberof X_1, x_0 \memberof X_0, x_1 \leq_X x_0 \}
	 \end{array}$
\end{center}
The point of view that we foster in this paper is that
summation and fibration are inverse transformations
between specification of formal contexts as matrices of relations
\[ {\cal X} = \left( \begin{array}{cc} {\cal X}_0 & \mu_{01} \\ \mu_{10} & {\cal X}_1 \end{array} \right) \]
and specification of formal contexts as indexed orders
$\morphism{{\oplus}{\cal X}}{t}{{\oplus}{\bf 1}}$.

Consider any map of formal contexts $\Morphism{{\cal Y}}{f}{{\cal X}}$.
By using the monotonicity of $f_0$ and $f_1$,
the defining conditions for the map $f = \pair{f_0}{f_1}$
and the orthogonality conditions satisfied by the target context,
it is straightforward to show that the sum function
$\longmorphism{Y_0 {+} Y_1}{{\oplus}f}{X_0 {+} X_1}$
defined by ${\oplus}f \define f_0 {+} f_1$
is a monotonic function between the sum orders
$\longmorphism{{\oplus}{\cal Y}}{{\oplus}f}{{\oplus}{\cal X}}$.
Since the monotonic index function 
$\morphism{{\oplus}{\cal X}}{\tau}{{\oplus}{\bf 1}}$
is the sum of the terminal map
$\tau = {\oplus}{\top_{\cal X}}$
(similarly for ${\cal Y}$),
the sum function ${\oplus}f$ preserves partitions,
in the sense that
${\oplus}{\top_{\cal Y}} = {\oplus}f \cdot {\oplus}{\top_{\cal X}}$.
In the reverse direction,
any monotonic function
$\morphism{{\oplus}{\cal Y}}{f}{{\oplus}{\cal X}}$
resolves into two functions $\morphism{Y_0}{f_0}{X_0}$ and $\morphism{Y_1}{f_1}{X_1}$
defined by pullback (restriction to inverse image) of the component orders $X_0$ and $X_1$.
The monotonicity of $f$
implies that $f_0$ and $f_1$ are monotonic functions
$\morphism{{\cal Y}_0}{f_0}{{\cal X}_0}$
and
$\morphism{{\cal Y}_1}{f_1}{{\cal X}_1}$,
and implies the defining conditions
$\nu_{01} \preceq f_0^{\triangleright} \circ \mu_{01} \circ f_1^{\triangleleft}$
and
$\nu_{10} \preceq f_1^{\triangleright} \circ \mu_{10} \circ f_0^{\triangleleft}$.
In summary,
the category of formal contexts is isomorphic to the category of orders over the binary order
\[ {\bf Cxt} \cong {\bf Ord}/\overline{\bf 2}. \]
Define a summation functor
$\longmorphism{{\bf Cxt}}{{\oplus}}{{\bf Ord}}$
to be the composite ${\oplus} \define {\cong} \cdot {\partial_0}$,
where
$\longmorphism{{\bf Ord}/\overline{\bf 2}}{\partial_0}{{\bf Ord}}$
is the projection functor
mapping monotonic functions $\morphism{{\cal X}}{t}{\overline{\bf 2}}$
to their source order ${\cal X}$.
Summation is left adjoint
\[ {\oplus} \dashv {\rm Inc} \]
to the inclusion functor $\longmorphism{{\bf Ord}}{{\rm Inc}}{{\bf Cxt}}$
with adjunction unit context map $\Morphism{{\cal X}}{i}{{\rm Inc}({\oplus}{\cal X})}$
consisting of the pair $i = \pair{i_0}{i_1}$ of sum injection functions.

For any formal context
\[ {\cal X} = \left( \begin{array}{cc} {\cal X}_0 & \mu_{01} \\ \mu_{10} & {\cal X}_1 \end{array} \right) \]
we can construct the Dedekind-MacNeille completion of the sum
${\bf DM}({\oplus}{\cal X})$
whose elements,
called concepts,
are formalized as extent-intent pairs
$\pair{[\phi_0,\phi_1]}{(\psi_0,\psi_1)}$:
the extent pair of order ideals satisfies the ideal assertions
$\mu_{01} \circ \phi_1 \;\preceq\; \phi_0$ and $\mu_{10} \circ \phi_0 \;\preceq\; \phi_1$
and the extent constraints
$\phi_0 \;=\; \looer{(\psi_0)}{{\cal X}_0} \wedge \deriveinv{(\psi_1)}{\mu_{01}}$
and
$\phi_1 \;=\; \deriveinv{(\psi_0)}{\mu_{10}} \wedge \looer{(\psi_1)}{{\cal X}_1}$;
and
the intent pair of order filters satisfies the filter assertions
$\psi_0 \circ \mu_{01} \;\preceq\; \psi_1$ and $\psi_1 \circ \mu_{10} \;\preceq\; \psi_0$
and the intent constraints
$\psi_0 \;=\; \upper{(\phi_0)}{{\cal X}_0} \wedge \derivedir{(\phi_1)}{\mu_{10}}$
and
$\psi_1 \;=\; \derivedir{(\phi_0)}{\mu_{01}} \wedge \upper{(\phi_1)}{{\cal X}_1}$.
The embedding generator functions
${\cal X}_0 \stackrel{\hat{\imath}_0}{\longrightarrow} {\bf DM}({\oplus}{\cal X}) \stackrel{\hat{\imath}_1}{\longleftarrow} {\cal X}_1$,
which are the compositions 
of coproduct injection followed by Dedekind-MacNeille completion,
are described in detail as the following concepts (extent-intent pairs):
\footnotesize
\begin{center}
	$\begin{array}{|c||c|c|c|c|}
		\hline
		{\bf generating} & \multicolumn{2}{c|}{{\bf extent}} & \multicolumn{2}{|c|}{{\bf intent}} \\
		\cline{2-5}
		{\bf element} & \phi_0 & \phi_1 & \psi_0 & \psi_1 \\
		\hline\hline
		x_0 \memberof {\cal X}_0
			& {\downarrow} x_0 \wedge \closure{(x_0)}{\mu}
				& \looer{(x_0{\mu})}{{\cal X}_1}
					& {\uparrow} x_0
						& x_0{\mu} \\
		\hline
		x_1 \memberof {\cal X}_1
			& {\mu}x_1
				& {\downarrow} x_1
					& \upper{({\mu}x_1)}{{\cal X}_0}
						& {\uparrow} x_1 \wedge \closure{(x_1)}{\mu} \\
		\hline
	 \end{array}$
\end{center}
\normalsize
The generated concept
${\imath_0}(x_0) =
	\left\langle
	\left( {\downarrow} x_0 \wedge \closure{(x_0)}{\mu}, \looer{(x_0{\mu})}{{\cal X}_1} \right),
	\left( {\uparrow} x_0, x_0{\mu} \right)
	\right\rangle$
can start from the pair
$\left( {\downarrow}x_0, \emptyset \right) \in \product{\powerof{{\cal X}_0^{\rm op}}}{\powerof{{\cal X}_1^{\rm op}}}$,
and the generated concept
${\imath_1}(x_1) =
	\left\langle
	\left( {\mu}x_1, {\downarrow} x_1 \right),
	\left( \upper{({\mu}x_1)}{{\cal X}_0}, {\uparrow} x_1 \wedge \closure{(x_1)}{\mu} \right)
	\right\rangle$
can start from the pair
$\left( \emptyset, {\uparrow}x_1 \right) \in \product{\powerof{{\cal X}_0}}{\powerof{{\cal X}_1}}$.
Note that
the ideal pair $\pair{{\uparrow} x_0}{x_0{\mu}}$
satisfies the ideal assertion
$({\uparrow} x_0) \circ \mu \preceq x_0{\mu}$,
and
the filter pair $\pair{{\mu}x_1}{{\downarrow} x_1}$
satisfies the filter assertion
$\mu \circ ({\downarrow} x_1) \preceq {\mu}x_1$.

\begin{Theorem} {\bf [Equivalence]}
	The concept lattice of a relation
	is (isomorphic to)
	the Dedekind-MacNeille completion (Figure~\ref{DedekindMacNeille})
	of the sum of the contextual closure of the relation:
	\[ {\bf CL}\triple{X_0}{X_1}{\mu} \cong {\bf DM}({\oplus}{\cal X}^{\mu}). \]
\end{Theorem}
\begin{proof}
	Assume that $\pair{\phi}{\psi}$ is a concept w.r.t.\ ${\bf CL}\triple{X_0}{X_1}{\mu}$,
	hence satisfying the closure identities
	$\phi = \deriveinv{\psi}{\mu}$ and $\psi = \derivedir{\phi}{\mu}$.
	We will show that
	$\pair{\relcopair{\phi}{\mu \tensorimplysource \phi}{}{}}{\relpair{\psi \tensorimplytarget \mu}{\psi}{}{}}$
	is a concept w.r.t.\ ${\bf DM}({\oplus}{\cal X}^{\mu})$.
	The isomorphism between concepts will then be
	\[ \pair{\phi}{\psi} \hspace{5mm} \longleftrightarrow \hspace{5mm} \pair{\relcopair{\phi}{\mu \tensorimplysource \phi}{}{}}{\relpair{\psi \tensorimplytarget \mu}{\psi}{}{}} \]
	We need to verify the closure identities
	\begin{center}
	\begin{tabular}{c@{\hspace{1cm}}c}
		$\begin{array}{r@{\;=\;}c@{\;=\;}l}
			\phi & \looer{(\psi \tensorimplytarget \mu)}{{\cal X}_0^\mu} \wedge \deriveinv{\psi}{\mu} & \deriveinv{\psi}{\mu} \\
			\mu \tensorimplysource \phi & \deriveinv{(\psi \tensorimplytarget \mu)}{{\neg}{\mu}} \wedge \looer{\psi}{{\cal X}_1^\mu} & \looer{\psi}{{\cal X}_1^\mu}
		 \end{array}$
		&
		$\begin{array}{r@{\;=\;}c@{\;=\;}l}
			\psi \tensorimplytarget \mu & \upper{\phi}{{\cal X}_0^\mu} \wedge \derivedir{(\mu \tensorimplysource \phi)}{{\neg}{\mu}} & \upper{\phi}{{\cal X}_0^\mu} \\
			\psi & \derivedir{\phi}{\mu} \wedge \upper{(\mu \tensorimplysource \phi)}{{\cal X}_1^\mu} & \derivedir{\phi}{\mu}
		 \end{array}$
	\end{tabular}
	\end{center}
	(see Appendix B).
	These hold,
	since
	$\mu \tensorimplysource \phi
	 = \mu \tensorimplysource \deriveinv{\psi}{\mu}
	 = \mu \tensorimplysource (\mu \tensorimplytarget \psi)
	 = (\mu \tensorimplysource \mu) \tensorimplytarget \psi
	 = {\cal X}_1^\mu \tensorimplytarget \psi
	 = \looer{\psi}{{\cal X}_1^\mu}$
	and
	$\psi \tensorimplytarget \mu 
	 = \derivedir{\phi}{\mu} \tensorimplytarget \mu 
	 = (\phi \tensorimplysource \mu) \tensorimplytarget \mu 
	 = \phi \tensorimplysource (\mu \tensorimplytarget \mu)
	 = \phi \tensorimplysource {\cal X}_0^\mu
	 = \upper{\phi}{{\cal X}_0^\mu}$.
	We also need to verify the assertions
	\begin{center}
	\begin{tabular}{c@{\hspace{1cm}}c}
		$\begin{array}{r@{\;\preceq\;}l}
			\mu \circ (\mu \tensorimplysource \phi) & \phi \\
			{\neg}\mu \circ \phi & (\mu \tensorimplysource \phi)
		 \end{array}$
		&
		$\begin{array}{r@{\;\preceq\;}l}
			(\psi \tensorimplytarget \mu) \circ \mu & \psi \\
			\psi \circ {\neg}\mu & (\psi \tensorimplytarget \mu)
		 \end{array}$
	\end{tabular}
	\end{center}
	These hold by modus ponens,
	and orthogonality constraints
	$\mu \circ \neg{\mu} \circ \phi \preceq \phi$
	and
	$\psi \circ \neg{\mu} \circ \mu \preceq \psi$.
\end{proof}

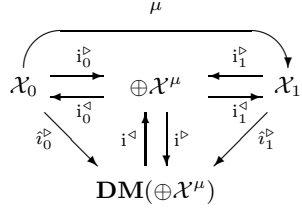
\begin{figure}
\begin{center}
	\begin{picture}(120,70)(-10,-40)
		\put(0,0){\makebox(0,0){\small${\cal X}_0$\normalsize}}
		\put(50,0){\makebox(0,0){\small${\oplus}{\cal X}^{\mu}$\normalsize}}
		\put(100,0){\makebox(0,0){\small${\cal X}_1$\normalsize}}
		\put(50,-40){\makebox(0,0){\small${\bf DM}({\oplus}{\cal X}^{\mu})$\normalsize}}
		\put(50,30){\makebox(0,0){\scriptsize$\mu$\normalsize}}
		\put(23,11){\makebox(0,0){\scriptsize${\rm i}_0^{\triangleright}$\normalsize}}
		\put(23,-10){\makebox(0,0){\scriptsize${\rm i}_0^{\triangleleft}$\normalsize}}
		\put(82,11){\makebox(0,0){\scriptsize${\rm i}_1^{\triangleright}$\normalsize}}
		\put(82,-10){\makebox(0,0){\scriptsize${\rm i}_1^{\triangleleft}$\normalsize}}
		\put(60,-20){\makebox(0,0){\scriptsize${\rm i}^{\triangleright}$\normalsize}}
		\put(40,-20){\makebox(0,0){\scriptsize${\rm i}^{\triangleleft}$\normalsize}}
		\put(8,-20){\makebox(0,0){\scriptsize$\hat{\imath}_0^\triangleright$\normalsize}}
		\put(92,-20){\makebox(0,0){\scriptsize$\hat{\imath}_1^\triangleright$\normalsize}}
		\put(10,4){\vector(1,0){20}}
		\put(90,4){\vector(-1,0){20}}
		\put(30,-4){\vector(-1,0){20}}
		\put(70,-4){\vector(1,0){20}}
		\put(54,-10){\vector(0,-1){20}}
		\put(46,-30){\vector(0,1){20}}
		\put(8,-10){\vector(1,-1){20}}
		\put(92,-10){\vector(-1,-1){20}}
		\put(50,7){\oval(100,28)[t]}
		\put(100,7){\vector(0,-1){0}}
	\end{picture}
\end{center}
\caption{{\bf Dedekind-MacNeille completion of the sum of the contextual closure} \label{DedekindMacNeille}}
\end{figure}

The sum order of the contextual closure of the original planetary relationship
is displayed in Table~\ref{sumclosure}.
This table is a matrix sum of the block in Table~\ref{planets} representing the original relationship,
the two blocks in Table~\ref{inducedorders} representing the induced orders,
and the block in Table~\ref{negrelplanets} representing the negation of the original planetary relationship.
\scriptsize
\begin{table}
\begin{center}	\begin{tabular}{|l|@{\hspace{1.0mm}}c@{\hspace{1.0mm}}c@{\hspace{1.0mm}}c@{\hspace{1.0mm}}c@{\hspace{1.0mm}}c@{\hspace{1.0mm}}c@{\hspace{1.0mm}}c@{\hspace{1.0mm}}c@{\hspace{1.0mm}}c@{\hspace{1.0mm}}|@{\hspace{1.0mm}}c@{\hspace{1.0mm}}c@{\hspace{1.0mm}}c@{\hspace{1.0mm}}c@{\hspace{1.0mm}}c@{\hspace{1.0mm}}c@{\hspace{1.0mm}}c@{\hspace{0.5mm}}|} \hline
		   & Me & V & E & Ma & J & S & U & N & P & ss & sm & sl & dn & df & my & mn \\ \hline
		Me & $\times$ & $\times$ &&&&&&&& $\times$ &&& $\times$ &&& $\times$   \\
		V  & $\times$ & $\times$ &&&&&&&& $\times$ &&& $\times$ &&& $\times$   \\
		E  &&& $\times$ & $\times$ &&&&&& $\times$ &&& $\times$ && $\times$ &  \\
		Ma &&& $\times$ & $\times$ &&&&&& $\times$ &&& $\times$ && $\times$ &  \\
		J  &&&&& $\times$ & $\times$ &&&&&& $\times$ && $\times$ & $\times$ &  \\
		S  &&&&& $\times$ & $\times$ &&&&&& $\times$ && $\times$ & $\times$ &  \\
		U  &&&&&&& $\times$ & $\times$ &&& $\times$ &&& $\times$ & $\times$ &  \\
		N  &&&&&&& $\times$ & $\times$ &&& $\times$ &&& $\times$ & $\times$ &  \\
		P  &&&&&&&&& $\times$ & $\times$ &&&& $\times$ & $\times$ &            \\ \hline
		ss &&&&&&&&&& $\times$ &&&&&&                                          \\
		sm &&&&&&& $\times$ & $\times$ &&& $\times$ &&& $\times$ & $\times$ &  \\
		sl &&&&& $\times$ & $\times$ &&&&&& $\times$ && $\times$ & $\times$ &  \\
		dn &&&&&&&&&& $\times$ &&& $\times$ &&&                                \\
		df &&&&&&&&&&&&&& $\times$ & $\times$ &                                \\
		my &&&&&&&&&&&&&&& $\times$ &                                          \\
		mn & $\times$ & $\times$ &&&&&&&& $\times$ &&& $\times$ &&& $\times$   \\ \hline
	\end{tabular}
\end{center}
\caption{{\bf Sum ${\oplus}{\cal X}^{\mu}$ of the contextual closure of the original planetary relationship} \label{sumclosure}}
\end{table}
\normalsize
The concepts in the Dedekind-MacNeille completion of the planetary context
(the sum of the contextual closure of the planetary relationship)
are listed in Table~\ref{DMplanets}.
A comparison of these concepts with
the concepts in the concept lattice of the planetary relationship as listed in Table~\ref{CLplanets}
will confirm their isormorphism.
\footnotesize
\begin{table}
\begin{center}
	\begin{tabular}{|c||c|c|c|c|}
		\hline
		{\bf concept} & \multicolumn{2}{c|}{{\bf extent}} & \multicolumn{2}{|c|}{{\bf intent}} \\
		\cline{2-5}
		{\bf description} & $\phi_0$ & $\phi_1$ & $\psi_0$ & $\psi_1$ \\
		\hline\hline
		``everything''
			& $X_0$
				 & $X_1$
					 & $\emptyset$
						 & $\emptyset$
		\\
		``with moon''
			& $\{ {\rm {\rm E}},{\rm Ma},{\rm J},{\rm S},{\rm U},{\rm N},{\rm P} \}$
				 & $\{ {\rm sm},{\rm sl},{\rm df},{\rm my} \}$
					 & $\emptyset$
						 & $\{ {\rm my} \}$
		\\
		``small size''
			& $\{ {\rm Me},{\rm V},{\rm E},{\rm Ma},{\rm P} \}$
				& $\{ {\rm ss},{\rm dn},{\rm mn} \}$
					& $\emptyset$
						& $\{ {\rm ss} \}$
		\\
		``small with moon''
			& $\{ {\rm E},{\rm Ma},{\rm P} \}$
				& $\emptyset$
					& $\emptyset$
						& $\{ {\rm ss},{\rm my} \}$
		\\
		``far''
			& $\{ {\rm J},{\rm S},{\rm U},{\rm N},{\rm P} \}$
				& $\{ {\rm sm},{\rm sl},{\rm df} \}$
					& $\emptyset$
						& $\{ {\rm df},{\rm my} \}$
		\\
		``near''
			& $\{ {\rm Me},{\rm V},{\rm E},{\rm Ma} \}$
				& $\{ {\rm dn},{\rm mn} \}$
					& $\emptyset$
						& $\{ {\rm ss},{\rm dn} \}$
		\\
		``Plutoness''
			& $\{ {\rm P} \}$
				& $\emptyset$
					& $\{ {\rm P} \}$
						& $\{ {\rm ss},{\rm df},{\rm my} \}$
		\\
		``medium size''
			& $\{ {\rm U},{\rm N} \}$
				& $\{ {\rm sm} \}$
					& $\{ {\rm U},{\rm N} \}$
						& $\{ {\rm sm},{\rm df},{\rm my} \}$
		\\
		``large size''
			& $\{ {\rm J},{\rm S} \}$
				& $\{ {\rm sl} \}$
					& $\{ {\rm J},{\rm S} \}$
						& $\{ {\rm sl},{\rm df},{\rm my} \}$
		\\
		``near with moon''
			& $\{ {\rm E},{\rm Ma} \}$
				& $\emptyset$
					& $\{ {\rm E},{\rm Ma} \}$
						& $\{ {\rm ss},{\rm dn},{\rm my} \}$
		\\
		``moonless''
			& $\{ {\rm Me},{\rm V} \}$
				& $\{ {\rm mn} \}$
					& $\{ {\rm Me},{\rm V} \}$
						& $\{ {\rm ss},{\rm dn},{\rm mn} \}$
		\\
		``nothing''
			& $\emptyset$
				& $\emptyset$
					& $X_0$
						& $X_1$
		 \\ \hline
	\end{tabular}
\end{center}
\caption{{\bf Dedekind-MacNeille completion ${\bf DM}({\oplus}{\cal X}^{\mu})$ for the sum} \label{DMplanets}}
\end{table}
\normalsize


\section{Future Work}

In a follow-up paper
we generalize formal contexts to a distributed version more suitable for knowledge representation called formal situations.
There we give an equivalent categorical rendition known as distributed orders.
In this follow-up paper
we incorporate Woods's notions of abstract conceptual descriptions, subsumption, extended quantifiers, etc.
At the same time we rationalize the assertional/terminological distinction --- A-boxes versus T-boxes.

In a second paper
we abstract formal situations and their concept construction
from the special order-theoretic realm
to the general realm of semiadditive Heyting categories.
In this abstraction,
formal situations and distributed orders become distributed monads,
and derivation operators become a kind of logical negation.
We resolve distributed monads into the two constructions of distribution (matrices) and monad-enrichment,
and
identify the direct/inverse derivation operators of concept construction
and the upper/lower operators of Dedekind-MacNeille completion
with the two dual implication operators from relational logic \cite{Lawvere,Kent}.
Since the structural aspect of the mathematics here is very close to a Grothendieck topos,
the topos nature of formal concept analysis needs to be investigated.

Initial applications have been carried out in terms of {\bf C}$+$$+$ software which implements the semantic version of formal context defined in this paper,
the modified approach to concept lattice construction,
and query processing against a lattice,
in a windows environment on a personal computer or a work station.
In a companion paper we have abstracted the related but distinct approach of Pawlak
to classification and predicate approximation using rough sets \cite{Pawlak}.
The mathematics shows an intimate connection between the two approaches.
Further work needs to be done on extending the new approach for concept construction
to conceptual scaling,
the situation more common in both object-orientation, database and knowledge representation,
where multi-valued attributes exist.


\newpage
\appendix

\section{The Central Adjointness of Logic}

A {\em closed relation\/} $\term{{\cal X}}{{\alpha}}{{\cal Y}}$
between two orders
is a binary relation $\alpha \subseteq \product{X}{Y}$ between the underlying sets
which is closed on the left w.r.t.\ ${\cal X}$ in the sense that
$x' \leq_X x$ and $x{\alpha}y$ implies $x'{\alpha}y$ for all $x',x \in X$ and $y \in Y$,
and closed on the right w.r.t.\ ${\cal Y}$ in the dual sense that
$x{\alpha}y$ and $y \leq_Y y'$ implies $x{\alpha}y'$ for all $x \in X$ and $y,y' \in Y$.
Clearly,
an alternate description is that a closed relation $\term{{\cal X}}{{\alpha}}{{\cal Y}}$
is a monotonic function to the special Boolean order
$\morphism{\product{{\cal X}^{\rm op}}{{\cal Y}}}{\alpha}{{\bf 2}}$.
Closed relations between ${\cal X}$ and ${\cal Y}$ are ordered by subset inclusion (homset order)
in the product powerset:
two closed relations $\term{{\cal X}}{{\alpha},{\beta}}{{\cal Y}}$ are ordered ${\alpha} \preceq {\beta}$
when ${\alpha} \subseteq {\beta}$ as subsets ${\alpha},{\beta} \in \wp(\product{X}{Y})$.
Any closed relation $\term{{\cal X}}{{\alpha}}{{\cal Y}}$
has
an {\em opposite\/} closed relation $\term{{\cal Y}^{\rm op}}{{\alpha}^{\rm op}}{{\cal X}^{\rm op}}$
defined by $y{\alpha}^{\rm op}x$ iff $x{\alpha}y$.
Two closed relations $\term{{\cal X}}{{\alpha}}{{\cal Y}}$ and $\term{{\cal Y}}{{\beta}}{{\cal Z}}$
with matching target and source, respectively, are said to be composable.
The composition of two such closed relations
is defined by
${\alpha} \circ {\beta} \define \{ (x,z) \mid (\exists y \in Y)\; x{\alpha}y \mbox{ and } y{\beta}z \}$
--- ordinary relational composition.
For every preorder ${\cal X} = \pair{X}{\leq_{\cal X}}$
the order relation ${\leq_{\cal X}} \subseteq \product{X}{X}$ is a closed relation
$\term{{\cal X}}{{\cal X}}{{\cal X}}$,
defined by $x{\cal X}x'$ when $x \leq_{\cal X} x'$,
which is the identity relation at ${\cal X}$ w.r.t.\ relational composition:
${\cal X} \circ r = r$ for any relation $\term{{\cal X}}{r}{{\cal Y}}$.
Composition preserves homset order.
The opposite operator is an involution.

For any monotonic function $\morphism{{\cal X}}{h}{{\cal Y}}$
there are two associated closed relations:
its {\em direct graph\/} $\term{{\cal X}}{h^\triangleright}{{\cal Y}}$
defined by $h^\triangleright \define \{ (x,y) \mid h(x) \leq_Y y \}$,
and
its {\em inverse graph\/} $\term{{\cal Y}}{h^\triangleleft}{{\cal X}}$
defined by $h^\triangleleft \define \{ (y,x) \mid y \leq_Y h(x) \}$.
These relations form an adjoint pair $h^\triangleright \dashv h^\triangleleft$.
The graph of a composite map
$\longmorphism{{\cal X}}{f \cdot g}{{\cal Z}}$
is the composition of the component graphs,
in
both the direct sense $(f \cdot g)^\triangleright = f^\triangleright \circ g^\triangleright$
and the inverse sense $(f \cdot g)^\triangleleft = g^\triangleleft \circ f^\triangleleft$.
The graph of the identity function $\longmorphism{{\cal X}}{{\rm Id}_X}{{\cal X}}$
is the identity relation,
in
both the direct sense ${\rm Id}_X^\triangleright = {\cal X}$
and the inverse sense ${\rm Id}_X^\triangleleft = {\cal X}$.

The central adjointness for both classical and intuitionistic logic
is the adjunction between conjunction and implication:
\begin{center}
	$\begin{array}{rclcccrcl}
		\mbox{conjunction} && \mbox{implication} \\
		(\;) \wedge p & \dashv & (\;) \Leftarrow p
	 \end{array}$
\end{center}
as verified by the adjointness equivalence
\[	q \wedge p \;\vdash\; r
	\hspace{.2in} \mbox{iff} \hspace{.2in}
	q \;\vdash\; r \Leftarrow p \]
for any three propositional symbols $p,q,r \in {\bf 2} = \{ 0,1 \}$.
One reason why this adjointness is so central is that it has several powerful analogs
in other contexts.
We are especially interested in the relational analog,
as illustrated in Table~\ref{analogies},
and the central adjointness in the relational context.
\begin{table}
\begin{center}
\begin{tabular}{|rc|rc|}
	\hline
	\multicolumn{2}{|c|}{{\sc Traditional Logic}} &
	\multicolumn{2}{c|}{{\sc Relational Logic}}
	\\ \hline\hline
	propositions & $p$               & relations   & $\term{{\cal X}}{{\alpha}}{{\cal Y}}$ \\
	conjunction  & $p \wedge q$      & composition & ${\alpha} \circ {\beta}$   \\
	entailment   & $p \vdash q$      & order       & ${\alpha} \preceq {\beta}$ \\
	implication  & $p \Rightarrow q$ & implication & ${\beta} \tensorimplysource {\alpha}$ \\ \hline
\end{tabular}
\end{center}
	\caption{Analogies between traditional and relational logic \label{analogies}}
\end{table}
To extend this analogy to implication,
we need a notion of relational implication.
Unlike the case of propositional logic
where the combining form of conjunction is symmetric,
in relational logic the combining form of composition is asymmetric.
This implies existence of two (related, but nonequivalent) relational implications
\begin{center}
$\begin{array}{rcl}
	{\alpha} \tensorimplytarget {\gamma}
	 & \define
	  & \{ (x,z) \mid (\forall y \in Y)\; z{\gamma}y \mbox{ implies } x{\alpha}y \}
	   \\
	{\beta} \tensorimplysource {\alpha}
	 & \define
	  & \{ (z,y) \mid (\forall x \in X)\; x{\beta}z \mbox{ implies } x{\alpha}y \}
 \end{array}$
\end{center}
for any three closed relations
$\term{{\cal X}}{{\beta}}{{\cal Z}}$,
$\term{{\cal Z}}{{\gamma}}{{\cal Y}}$ and $\term{{\cal X}}{{\alpha}}{{\cal Y}}$.
These implications are closed relations
$\term{{\cal X}}{{\alpha} \tensorimplytarget {\gamma}}{{\cal Z}}$
and
$\term{{\cal Z}}{{\beta} \tensorimplysource {\alpha}}{{\cal Y}}$.
The relationship between the two forms of implication
is expressible in terms of the opposite involution
\[ {\alpha} \tensorimplytarget {\gamma} = {({\gamma}^{\rm op} \tensorimplysource {\alpha}^{\rm op})}^{\rm op}
   \hspace{.2in} \mbox{and} \hspace{.2in}
   {\beta} \tensorimplysource {\alpha} = {({\alpha}^{\rm op} \tensorimplytarget {\beta}^{\rm op})}^{\rm op}. \]
The central adjointness for relational logic
is
the adjunction between relational composition and relational implication:
\begin{center}
	$\begin{array}{rclcccrcl}
		\shortstack{relational\\composition} && \shortstack{relational\\implication} \\
		{\beta} \circ (\;) & \dashv & {\beta} \tensorimplysource (\;) \\
		(\;) \circ {\gamma} & \dashv & (\;) \tensorimplytarget {\gamma}
	 \end{array}$
\end{center}
as verified by the adjointness equivalences
\[  {\gamma} \;\preceq\; {\beta} \tensorimplysource {\alpha}
	\hspace{.2in} \mbox{iff} \hspace{.2in}
	{\beta} \circ {\gamma} \;\preceq\; {\alpha}
	\hspace{.2in} \mbox{iff} \hspace{.2in}
	{\beta} \;\preceq\; {\alpha} \tensorimplytarget {\gamma}. \]

\subsection{Upper/Lower Operators}

A {\em closed-below subset (order ideal)\/} of an order ${\cal X} = \pair{X}{\leq}$
is a subset $\phi \subseteq X$ satisfying the condition:
$x' \leq_X x$ and $x \in \phi$ implies $x' \in \phi$ for all $x,x' \in X$.
Clearly,
an order ideal can be expressed alternately as a closed relation $\term{{\cal X}}{\phi}{{\bf 1}}$
or as a monotonic function $\morphism{{\cal X}^{\rm op}}{\phi}{{\bf 2}}$.
So the set of order ideals of ${\cal X}$ is the special exponential order
$\powerof{{\cal X}^{\rm op}}$.
Dually,
a {\em closed-above subset (order filter)\/} of ${\cal X}$
is a subset $B \subseteq X$ satisfying the condition:
$x \in B$ and $x \leq_X x'$ implies $x' \in B$ for all $x,x' \in X$.
An order filter is a closed relation $\term{{\bf 1}}{B}{{\cal X}}$
or a monotonic function $\morphism{{\cal X}}{\psi}{{\bf 2}}$.
So the set of order filters of ${\cal X}$ is the special exponential order
$\powerof{{\cal X}}$.
If $\term{{\cal X}}{\phi}{{\bf 1}}$ is an order ideal
and $\term{{\bf 1}}{\psi}{{\cal X}}$ is an order filter,
then the relational composition
$(\term{{\cal X}}{\phi \circ \psi}{{\cal X}}) = \product{\phi}{\psi}$
is the Cartesian product of $\phi$ and $\psi$ as an endorelation.
The relational interior of $\phi \circ \psi$
is the intersection of $A$ and $\psi$ as a relational comonoid.

Given any order ${\cal X}$,
the {\em upper operator\/} $\upper{(\;)}{{\cal X}}$
maps any subset $\phi \subseteq X$ to its subset of upper bounds
\begin{center}
	$\begin{array}{r@{\hspace{0.4em}}c@{\hspace{0.4em}}l}
		\upper{\phi}{{\cal X}} & \define & \{ x_2 \memberof X \mid x_1 \leq x_2 \mbox{ for all } x_1 \in \phi \} \\
		                    & =       & \{ x_2 \memberof X \mid (\forall x_1 \memberof X)\; x_1 \in \phi \mbox{ implies } x_1 \leq x_2 \} \\
		                    & =       & \{ x_2 \memberof X \mid \phi \subseteq {\downarrow}x_2 \} \\
		                    & =       & \bigcup_{x_1 \in \phi} {\uparrow} x_1,
	 \end{array}$
\end{center}
which is an order filter.
The upper operator is invariant w.r.t.\ closure below,
$\upper{({\downarrow}{\phi})}{{\cal X}} = \upper{\phi}{{\cal X}}$
for any subset $\phi \subseteq X$,
and hence we can restrict application of the operator to just order ideals
$\phi \memberof \powerof{{\cal X}^{\rm op}}$,
so that
$\longmorphism{\powerof{{\cal X}^{\rm op}}}{\upper{(\,)}{{\cal X}}}{{(\powerof{{\cal X}})}^{\rm op}}$.
It is important to observe the fact that
the upper operator is a special case of relational implication
\[ \upper{\phi}{{\cal X}} = \phi \tensorimplysource {\cal X}. \]
The varying quantity $\phi$ is in the contravariant position w.r.t.\ implication.
Dually,
the {\em lower operator\/} $\looer{(\;)}{{\cal X}}$
maps any subset $\psi \subseteq X$ to its subset of lower bounds
\begin{center}
	$\begin{array}{r@{\hspace{0.4em}}c@{\hspace{0.4em}}l}
		\looer{\psi}{{\cal X}} & \define & \{ x_1 \memberof X \mid x_1 \leq x_2 \mbox{ for all } x_2 \in \psi \} \\
		                    & =       & \{ x_1 \memberof X \mid (\forall x_2 \memberof X)\; x_2 \in \psi \mbox{ implies } x_1 \leq x_2 \} \\
		                    & =       & \{ x_1 \memberof X \mid \psi \subseteq {\uparrow} x_2 \} \\
		                    & =       & \bigcap_{x_2 \in \psi} {\downarrow} x_2,
	 \end{array}$
\end{center}
an order ideal,
is invariant w.r.t.\ closure above,
$\looer{({\uparrow}{\psi})}{{\cal X}} = \looer{\psi}{{\cal X}}$
for any subset $\psi \subseteq X$,
and hence when restricted to just order filters
$\psi \powerof{{\cal X}}$
has type
$\longmorphism{{(\powerof{{\cal X}})}^{\rm op}}{\upper{(\,)}{{\cal X}}}{\powerof{{\cal X}^{\rm op}}}$.
The lower operator is also a special case of relational implication
\[ \looer{\psi}{{\cal X}} = {\cal X} \tensorimplytarget \psi. \]
The upper and lower operators form an adjoint pair of monotonic functions
\[ \upper{(\;)}{{\cal X}} \dashv \looer{(\;)}{{\cal X}}. \]
These operators can be expressed in terms of elements as follows:
\[
	\upper{\phi}{{\cal X}} = \bigcap \{ {\uparrow}x   \mid x \in \phi \} = \bigcap \{ \upper{x}{{\cal X}} \mid x \in \phi \}
	\hspace{1cm}
	\looer{\psi}{{\cal X}} = \bigcap \{ {\downarrow}y \mid y \in \psi \} = \bigcap \{ \looer{y}{{\cal X}} \mid y \in \psi \}
\]

\subsection{Direct/Inverse Derivation}

Let $\term{{\cal X}_0}{\mu}{{\cal X}_1}$ by any closed relation.
The {\em direct derivation (or intent) operator\/} $\derivedir{(\;)}{\mu}$
maps any subset $\phi \subseteq X_0$ to the subset
\begin{center}
	$\begin{array}{r@{\hspace{0.4em}}c@{\hspace{0.4em}}l}
		\derivedir{\phi}{\mu} & \define & \{ x_1 \memberof X_1 \mid x_0{\mu}x_1 \mbox{ for all } x_0 \memberof \phi \} \\
		                         & =       & \{ x_1 \memberof X_1 \mid (\forall x_0 \memberof X_0)\; x_0 \memberof \phi \mbox{ implies } x_0{\mu}x_1 \},
	 \end{array}$
\end{center}
an order filter of ${\cal X}_1$.
The direct derivation operator is invariant w.r.t.\ closure below,
$\derivedir{({\downarrow}{\phi})}{\mu} = \derivedir{\phi}{\mu}$
for any subset $\phi \subseteq X_0$,
and hence we can restrict application of the operator to
just order ideals $\phi \memberof \powerof{{\cal X}_0^{\rm op}}$,
so that
$\longmorphism{\powerof{{\cal X}_0^{\rm op}}}{\derivedir{(\,)}{\mu}}{{(\powerof{{\cal X}_1})^{\rm op}}}$.
The intuitive interpretation in terms of formal concept analysis is that $\derivedir{\phi}{\mu}$
is the collection of all {\em individual\/} attributes
that the entities in $\phi$ share, or have in common.
It is important to observe the fact that,
just as for the upper operator in the more homogeneous case of a single poset,
the direct derivation operator is also a special case of relational implication
\[ \derivedir{\phi}{\mu} = \phi \tensorimplysource {\mu}. \]
Dually,
the {\em inverse derivation (or extent) operator\/} $\deriveinv{(\;)}{\mu}$
maps any subset $\psi \subseteq X_1$ to the subset
\begin{center}
	$\begin{array}{r@{\hspace{0.4em}}c@{\hspace{0.4em}}l}
		\deriveinv{\psi}{\mu} & \define & \{ x_0 \memberof X_0 \mid x_0{\mu}x_1 \mbox{ for all } x_1 \memberof \psi \} \\
		                         & =       & \{ x_0 \memberof X_0 \mid (\forall x_1 \memberof X_1)\; x_1 \memberof \psi \mbox{ implies } x_0{\mu}x_1 \},
	 \end{array}$
\end{center}
an order ideal of ${\cal X}_0$,
is invariant w.r.t.\ closure above,
$\deriveinv{({\uparrow}{\psi})}{\mu} = \deriveinv{\psi}{\mu}$
for any subset $\psi \subseteq X_1$,
and hence when restricted to just order filters
$\psi \memberof \powerof{{\cal X}_1}$
has type
$\longmorphism{(\powerof{{\cal X}_1})^{\rm op}}{\deriveinv{(\,)}{\mu}}{\powerof{{\cal X}_0^{\rm op}}}$.
The intuitive interpretation in terms of formal concept analysis is that $\deriveinv{\psi}{\mu}$
is the collection of all {\em individual\/} entities that share the attributes in $\psi$.
The inverse derivation operator is also a special case of relational implication
\[ \deriveinv{\psi}{\mu} = {\mu} \tensorimplytarget \psi. \]
The direct and the inverse derivation operator
form an adjoint pair of monotonic functions
\[ \derivedir{(\,)}{\mu} \dashv \deriveinv{(\,)}{\mu}.\]


\section{Properties of Derivation}

The following facts concerning sum orders
are useful in determining derivation along pairings and copairings.
\begin{enumerate}
	\item For any relations
		$\term{{\cal W}}{x}{{\cal X}_0}, \term{{\cal W}}{y}{{\cal X}_1}$ and $\term{{\cal W}}{z}{{\cal Z}}$
		\[ z \tensorimplysource \relpair{x}{y}{}{} = \relpair{z \tensorimplysource x}{z \tensorimplysource y}{}{} \]
	\item For any relations
		$\term{{\cal Z}}{x}{{\cal X}_0}, \term{{\cal Z}}{y}{{\cal X}_1}, \term{{\cal W}}{x'}{{\cal X}_0}$ and $\term{{\cal W}}{y'}{{\cal X}_1}$
		\[ \relpair{x'}{y'}{}{} \tensorimplytarget \relpair{x}{y}{}{} = (x' \tensorimplytarget x) \wedge (y' \tensorimplytarget y) \]
	\item For any relations
		$\term{{\cal X}_0}{x}{{\cal Z}}, \term{{\cal X}_1}{y}{{\cal Z}}, \term{{\cal X}_0}{x'}{{\cal W}}$ and $\term{{\cal X}_1}{y'}{{\cal W}}$
		\[ \relcopair{x}{y}{}{} \tensorimplysource \relcopair{x'}{y'}{}{} = (x \tensorimplysource x') \wedge (y \tensorimplysource y') \]
	\item For any relations
		$\term{{\cal X}_0}{x}{{\cal Z}}, \term{{\cal X}_1}{y}{{\cal Z}}$ and $\term{{\cal W}}{z}{{\cal Z}}$
		\[ \relcopair{x}{y}{}{} \tensorimplytarget z = \relcopair{x \tensorimplytarget z}{y \tensorimplytarget z}{}{} \]
\end{enumerate}
The following facts concerning relational implication
are useful in determining derivation along inverse image relations.
\begin{enumerate}
	\item For any monotonic function $\morphism{{\cal Y}}{f}{{\cal X}}$
		and any relations
		$\term{{\cal Y}}{\sigma}{{\cal W}}$ and $\term{{\cal X}}{\rho}{{\cal Z}}$
		\[ \sigma \tensorimplysource (f^\triangleright \circ \rho) = (f^\triangleleft \circ \sigma) \tensorimplysource \rho \]
	\item For any monotonic function $\morphism{{\cal Y}}{f}{{\cal X}}$
		and any relations
		$\term{{\cal Z}}{\sigma}{{\cal W}}$ and $\term{{\cal Z}}{\rho}{{\cal X}}$
		\[ \sigma \tensorimplysource (\rho \circ f^\triangleleft) = (\sigma \tensorimplysource \rho) \circ f^\triangleleft \]
\end{enumerate}
Basic properties of derivation are classified according to
(1) type: order/continuity versus structure, and
(2) varying quantity: ideal/filter versus relation.
\begin{enumerate}
	\item {\bf Order/Continuity:}
		\begin{enumerate}
			\item {\bf Ideal/Filter:}
				\begin{description}
					\item[Order] Derivation, either direct or inverse,
						is contravariant in the subset argument
						\begin{center}\mbox{ }\hfill
							$\phi_1 \preceq_{\subpowerof{{\cal X}_0^{\rm op}}} \phi_2   \mbox{ implies } \derivedir{(\phi_1)}{\mu} \succeq_{\subpowerof{{\cal X}_1}} \derivedir{(\phi_2)}{\mu}$
						\hfill
							$\psi_1 \preceq_{\subpowerof{{\cal X}_1}} \psi_2 \mbox{ implies } \deriveinv{(\psi_1)}{\mu} \succeq_{\subpowerof{{\cal X}_0^{\rm op}}} \deriveinv{(\psi_2)}{\mu}$.
						\hfill\mbox{ }\end{center}
					\item[Continuity] Derivation is continuous in the subset argument
						\begin{center}\mbox{ }\hfill
							$\derivedir{(\bigvee_{i \in I} \phi_i)}{\mu} \;=\; \bigwedge_{i \in I} \derivedir{(\phi_i)}{\mu}$
						\hfill
							$\deriveinv{(\bigvee_{j \in J} \psi_j)}{\mu} \;=\; \bigwedge_{j \in J} \deriveinv{(\psi_j)}{\mu}$.
						\hfill\mbox{ }\end{center}
						In particular,
						when the index set is empty $I = \emptyset$ derivation is
						\begin{center}\mbox{ }\hfill
							$\derivedir{(\bot_{\subpowerof{{\cal X}_0^{\rm op}}})}{\mu}   = \derivedir{\emptyset}{\mu} \;=\; X_1 = \top_{\subpowerof{{\cal X}_1}}$
						\hfill
							$\deriveinv{(\bot_{\subpowerof{{\cal X}_1}})}{\mu} = \deriveinv{\emptyset}{\mu} \;=\; X_0 = \top_{\subpowerof{{\cal X}_0^{\rm op}}}$
						\hfill\mbox{ }\end{center}
						and when the index set is two $I = 2$ derivation is
						\begin{center}\mbox{ }\hfill
							$\derivedir{(\phi_1 \vee_{\subpowerof{{\cal X}_0^{\rm op}}} \phi_2)}{\mu} \;=\; \derivedir{\phi_1}{\mu} \wedge_{\subpowerof{{\cal X}_1}} \derivedir{\phi_2}{\mu}$
						\hfill
							$\deriveinv{(\psi_1 \vee_{\subpowerof{{\cal X}_1}} \psi_2)}{\mu} \;=\; \deriveinv{\psi_1}{\mu} \wedge_{\subpowerof{{\cal X}_0^{\rm op}}} \deriveinv{\psi_2}{\mu}$.
						\hfill\mbox{ }\end{center}
				\end{description}
			\item {\bf Relation:}
				\begin{description}
					\item[Order] Derivation, either direct or inverse,
						is covariant in the relation argument:
						for any two parallel contexts
						$\term{{\cal X}_0}{\mu,\mu'}{{\cal X}_1}$
						which are order as ${\mu} \preceq {\mu}'$,
						derivation is ordered by
						\begin{center}\mbox{ }\hfill
							$\derivedir{\phi}{\mu} \preceq_{\subpowerof{{\cal X}_1}} \derivedir{\phi}{\mu'}$
						\hfill
							$\deriveinv{\psi}{\mu} \preceq_{\subpowerof{{\cal X}_0^{\rm op}}} \deriveinv{\psi}{\mu'}$
						\hfill\mbox{ }\end{center}
					\item[Continuity] Derivation is continuous in the relational argument:
						for any collection of parallel contexts
						$\{ \term{{\cal X}_0}{\mu_i}{{\cal X}_1} \mid i \memberof I \}$
						derivation along the meet is expressed as
						\begin{center}\mbox{ }\hfill
							$\derivedir{\phi}{\wedge_{i \in I} \mu_i} \;=\; \bigwedge_{\subpowerof{{\cal X}_1}} \{ \derivedir{\phi}{\mu_i} \mid i \in I \}$
						\hfill
							$\deriveinv{\psi}{\wedge_{i \in I} \mu_i} \;=\; \bigwedge_{\subpowerof{{\cal X}_0^{\rm op}}} \{ \deriveinv{\psi}{\mu_i} \mid i \in I \}$.
						\hfill\mbox{ }\end{center}
						In particular,
						when the index set is empty $I = \emptyset$,
						and the meet is the top relation
						$\term{{\cal X}_0}{\top}{{\cal X}_1}$ defined by $\top = \product{X_0}{X_1}$,
						derivation is
						\begin{center}\mbox{ }\hfill
							$\derivedir{\phi}{{\top}} \;=\; X_1 = \top_{\subpowerof{{\cal X}_1}}$
						\hfill
							$\deriveinv{\psi}{{\top}} \;=\; X_0 = \top_{\subpowerof{{\cal X}_0^{\rm op}}}$
						\hfill\mbox{ }\end{center}
						and when the index set is two $I = 2$,
						and the meet is the relation
						$\term{{\cal X}_0}{{\mu} \wedge {\mu}'}{{\cal X}_1}$,
						derivation is
						\begin{center}\mbox{ }\hfill
							$\derivedir{\phi}{{{\mu} \wedge {\mu}'}} \;=\; \derivedir{\phi}{\mu} \wedge_{\subpowerof{{\cal X}_1}} \derivedir{\phi}{{\mu}'}$
						\hfill
							$\deriveinv{\psi}{{{\mu} \wedge {\mu}'}} \;=\; \deriveinv{\psi}{\mu} \wedge_{\subpowerof{{\cal X}_0^{\rm op}}} \deriveinv{\psi}{{\mu}'}$.
						\hfill\mbox{ }\end{center}
				\end{description}
		\end{enumerate}
	\item {\bf Structure:} Some general constructions on closed relations are listed below,
		along with the expression of their derivation operators in terms of
		component derivation and basic relational operators.
		\begin{enumerate}
			\item {\bf Ideal/Filter:}
				\begin{description}
					\item[Generators] Derivation, either direct or inverse,
						can be generated (constructed) from elements using intersection
						\begin{center}\mbox{ }\hfill
							$\begin{array}{r@{\;=\;}l}
								\derivedir{\phi}{\mu} 
									& \bigwedge_{\subpowerof{{\cal X}_1}} \{ x_0{\mu} \mid x_0 \memberof \phi \} \\
									& \bigwedge_{\subpowerof{{\cal X}_1}} \{ \derivedir{(x_0)}{\mu} \mid x_0 \memberof \phi \}
							\end{array}$
						\hfill
							$\begin{array}{r@{\;=\;}l}
								\deriveinv{\psi}{\mu}
									& \bigwedge_{\subpowerof{{\cal X}_0^{\rm op}}}   \{ {\mu}x_1 \mid x_1 \memberof \psi \} \\
									& \bigwedge_{\subpowerof{{\cal X}_0^{\rm op}}}   \{ \deriveinv{(x_1)}{\mu} \mid x_1 \memberof \psi \}.
							\end{array}$
						\hfill\mbox{ }\end{center}
				\end{description}
			\item {\bf Relation:}
				\begin{description}
					\item[Identity] Derivation across the {\em identity\/} relation
						$\term{{\cal X}}{{\cal X}}{{\cal X}}$
						reduces to the upper/lower operators
						\begin{center}\mbox{ }\hfill
							$\derivedir{\phi}{{\cal X}} \;=\; \upper{\phi}{{\cal X}}$
						\hfill
							$\deriveinv{\psi}{{\cal X}} \;=\; \looer{\psi}{{\cal X}}$,
						\hfill\mbox{ }\end{center}
						showing Dedekind-MacNeille completion to be a special case of concept-lattice construction.
					\item[Map] Derivation along the {\em direct graph\/}
						$\term{{\cal Y}}{f^\triangleright}{{\cal X}}$
						of a monotonic function
						$\morphism{{\cal Y}}{f}{{\cal X}}$
						factors as direct/inverse ideal image and upper/lower operator
						\begin{flushleft}$\begin{array}{r@{\hspace{2mm}}l}
						{\bf direct} &
						\left\{\mbox{\begin{tabular}{c}
							``the upper bounds of the direct ideal image'' \\
							$\derivedir{\psi}{f^\triangleright}
								= \psi \tensorimplysource f^\triangleright
								= ((f^\triangleleft \circ \psi) \tensorimplysource {\cal X})
								= \upper{(f^\triangleleft \circ \psi)}{{\cal X}}$
						\end{tabular}}\right. \\ \\
						{\bf inverse} &
						\left\{\mbox{\begin{tabular}{c}
							``the inverse ideal image of the lower bounds'' \\
							$\deriveinv{\phi}{f^\triangleright}
								= f^\triangleright \tensorimplytarget \phi
								= f^\triangleright \circ ({\cal X} \tensorimplytarget \phi)
								= f^\triangleright \circ \looer{\phi}{{\cal X}}$
						\end{tabular}}\right.
						\end{array}$\end{flushleft}

						Derivation along the {\em inverse graph\/}
						$\term{{\cal X}_1}{f^\triangleleft}{{\cal X}_0}$
						of a monotonic function
						$\morphism{{\cal Y}}{f}{{\cal X}}$
						factors as direct/inverse filter image and upper/lower operator
						\begin{flushleft}$\begin{array}{r@{\hspace{2mm}}l}
						{\bf direct} &
						\left\{\mbox{\begin{tabular}{c}
							``the inverse filter image of the upper bounds'' \\
							$\derivedir{\phi}{f^\triangleleft}
								= \phi \tensorimplysource f^\triangleleft
								= (\phi \tensorimplysource {\cal X}) \circ f^\triangleleft
								= \upper{\phi}{{\cal X}} \circ f^\triangleleft$
						\end{tabular}}\right. \\ \\
						{\bf inverse} &
						\left\{\mbox{\begin{tabular}{c}
							``the lower bounds of the direct filter image'' \\
							$\deriveinv{\psi}{f^\triangleleft}
								= f^\triangleleft \tensorimplytarget \psi
								= {\cal X} \tensorimplytarget (\psi \circ f^\triangleright)
								= \looer{(\psi \circ f^\triangleright)}{{\cal X}}$
						\end{tabular}}\right.
						\end{array}$\end{flushleft}
					\item[Inverse Image] Given any pair of monotonic functions
						$\morphism{{\cal Y}_0}{f_0}{{\cal X}_0}$
						and
						$\morphism{{\cal Y}_1}{f_1}{{\cal X}_1}$
						derivation along the {\em inverse image\/}
						$\term{{\cal Y}_0}{f_0^\triangleright \circ \mu \circ f_1^\triangleleft}{{\cal Y}_1}$
						of a relation $\term{{\cal X}_0}{\mu}{{\cal X}_1}$
						factors as 
						\begin{flushleft}$\begin{array}{r@{\hspace{2mm}}l}
						{\bf direct} &
						\left\{\mbox{\begin{tabular}{c}
							``direct ideal $f_0$-image -- direct $\mu$-derivation -- inverse filter $f_1$-image'' \\
							$\derivedir{\psi_0}{(f_0^\triangleright \circ \mu \circ f_1^\triangleleft)}
								= \psi_0 \tensorimplysource (f_0^\triangleright \circ \mu \circ f_1^\triangleleft)
								= ((f_0^\triangleleft \circ \psi_0) \tensorimplysource \mu) \circ f_1^\triangleleft
								= \derivedir{(f_0^\triangleleft \circ \psi_0)}{\mu} \circ f_1^\triangleleft$
						\end{tabular}}\right. \\ \\
						{\bf inverse} &
						\left\{\mbox{\begin{tabular}{c}
							``direct filter $f_1$-image -- inverse $\mu$-derivation -- inverse ideal $f_0$-image'' \\
							$\deriveinv{\psi_1}{(f_0^\triangleright \circ \mu \circ f_1^\triangleleft)}
								= (f_0^\triangleright \circ \mu \circ f_1^\triangleleft) \tensorimplytarget \psi_1
								= f_0^\triangleright \circ (\mu \tensorimplytarget (\psi_1 \circ f_1^\triangleright))
								= f_0^\triangleright \circ \deriveinv{(\psi_1 \circ f_1^\triangleright)}{\mu}$
						\end{tabular}}\right.
						\end{array}$\end{flushleft}
						Derivation along direct/inverse graphs of a monotonic function
						are special cases of inverse image.
					\item[Negation] Derivation along the {\em source negation\/} of a relation
						$\term{{\cal X}_1}{\mu \tensorimplysource {\cal X}_0}{{\cal X}_0}$
						factors as existential/universal ideal quantification and upper/lower operator
						\begin{flushleft}$\begin{array}{r@{\hspace{2mm}}l}
						{\bf direct} &
						\left\{\mbox{\begin{tabular}{c}
							``the upper bounds of the existential ideal quantification'' \\
							$\derivedir{\phi}{(\mu \tensorimplysource {\cal X}_0)}
								= \phi \tensorimplysource (\mu \tensorimplysource {\cal X}_0)
								= (\mu \circ \phi) \tensorimplysource {\cal X}_0
								= \upper{(\mu \circ \phi)}{{\cal X}_0}$
						\end{tabular}}\right. \\ \\
						{\bf inverse} &
						\left\{\mbox{\begin{tabular}{c}
							``the universal ideal quantification of the lower bounds'' \\
							$\deriveinv{\psi}{({\cal X}_1 \tensorimplysource \mu)}
								= (\mu \tensorimplysource {\cal X}_0) \tensorimplytarget \psi
								= \mu \tensorimplysource ({\cal X}_0 \tensorimplytarget \psi)
								= \mu \tensorimplysource \looer{\psi}{{\cal X}_0}$
						\end{tabular}}\right.
						\end{array}$\end{flushleft}

						Derivation along the {\em target negation\/} of a relation
						$\term{{\cal X}_1}{{\cal X}_1 \tensorimplytarget \mu}{{\cal X}_0}$
						factors as existential/universal filter quantification and upper/lower operator
						\begin{flushleft}$\begin{array}{r@{\hspace{2mm}}l}
						{\bf direct} &
						\left\{\mbox{\begin{tabular}{c}
							``the universal filter quantification of the upper bounds'' \\
							$\derivedir{\phi}{({\cal X}_1 \tensorimplytarget \mu)}
								= \phi \tensorimplysource ({\cal X}_1 \tensorimplytarget \mu)
								= (\phi \tensorimplysource {\cal X}_1) \tensorimplytarget \mu
								= (\upper{\phi}{{\cal X}_1}) \tensorimplytarget \mu$
						\end{tabular}}\right. \\ \\
						{\bf inverse} &
						\left\{\mbox{\begin{tabular}{c}
							``the lower bounds of the existential filter quantification'' \\
							$\deriveinv{\psi}{({\cal X}_1 \tensorimplytarget \mu)}
								= ({\cal X}_1 \tensorimplytarget \mu) \tensorimplytarget \psi
								= {\cal X}_1 \tensorimplytarget (\psi \circ \mu)
								= \looer{(\psi \circ \mu)}{{\cal X}_1}$
						\end{tabular}}\right.
						\end{array}$\end{flushleft}

						Derivation along the total {\em negation\/} of a relation
						$\term{{\cal X}_1}{{\neg}{\mu}}{{\cal X}_0}$
						is the meet
						\begin{center}\mbox{ }\hfill
							$\derivedir{\phi}{{\neg}{\mu}} \;=\; \upper{(\mu \circ \phi)}{{\cal X}_0} \wedge (\upper{\phi}{{\cal X}_1} \tensorimplytarget \mu)$
						\hfill
							$\deriveinv{\psi}{{\neg}{\mu}} \;=\; (\mu \tensorimplysource \looer{\psi}{{\cal X}_0}) \wedge \looer{(\psi \circ \mu)}{{\cal X}_1}$.
						\hfill\mbox{ }\end{center}
					\item[Pairing] Derivation along a {\em pairing\/} of two relations
						$\term{{\cal X}_0}{({\mu},{\nu})}{{\cal Y}{+}{\cal Z}}$
						is the pairing of filters and the meet of ideals
						\begin{center}\mbox{ }\hfill
							$\derivedir{\phi}{({\mu},{\nu})} \;=\; \left(\derivedir{\phi}{\mu},\derivedir{\phi}{\nu}\right)$
						\hfill
							$\deriveinv{(\psi,\zeta)}{({\mu},{\nu})} \;=\; \deriveinv{\psi}{\mu} \wedge_{\subpowerof{{\cal X}_0^{\rm op}}} \deriveinv{\zeta}{\nu}$.
						\hfill\mbox{ }\end{center}
					\item[CoPairing] Derivation along a {\em copairing\/} of two relations
						$\term{{\cal Y}{+}{\cal Z}}{[{\mu},{\nu}]}{{\cal X}_1}$
						is the meet of filters and the pairing of ideals
						\begin{center}\mbox{ }\hfill
							$\derivedir{[\phi,\theta]}{[{\mu},{\nu}]} \;=\; \derivedir{\phi}{\mu} \wedge_{\subpowerof{{\cal X}_1}} \derivedir{\theta}{\nu}$
						\hfill
							$\deriveinv{\psi}{[{\mu},{\nu}]} \;=\; \left[\deriveinv{\psi}{\mu},\deriveinv{\psi}{\nu}\right]$.
						\hfill\mbox{ }\end{center}
				\end{description}
		\end{enumerate}
\end{enumerate}

\begin{table}
\begin{center}
$\begin{array}{|cc|r@{\;=\;}l|r@{\;=\;}l|}
	\hline
	\multicolumn{2}{|c|}{\mbox{\bf relation}}
	 & \multicolumn{2}{|c|}{\mbox{\bf direct derivation operator}}
	  & \multicolumn{2}{|c|}{\mbox{\bf inverse derivation operator}}
	   \\
	\multicolumn{2}{|c|}{\term{{\cal X}_0}{\mu}{{\cal X}_1}}
	 & \multicolumn{2}{|c|}{\derivedir{\phi}{\mu} = \phi \tensorimplysource {\mu}}
	  & \multicolumn{2}{|c|}{\deriveinv{\psi}{\mu} = {\mu} \tensorimplytarget \psi}
	   \\ \hline\hline

	\mbox{\bf equality} & \term{{\cal X}}{{\cal X}}{{\cal X}}
	 & \derivedir{\phi}{{=}} & \upper{\phi}{{=}}
	  & \deriveinv{\psi}{{=}} & \looer{\psi}{{=}}
	  \\
	\mbox{\bf identity} & {\cal X} = \pair{X}{=}
	&& \{ x_1 \memberof X \mid x_0 = x_1 \mbox{ for all } x_0 \memberof \phi \}
	  && \{ x_0 \memberof X \mid x_0 = x_1 \mbox{ for all } x_1 \memberof \psi \}
	  \\
	&&& \left\{ \begin{array}{ll}
                   \emptyset, & \mbox{ if } |\phi| > 1 \\
                   \phi,         & \mbox{ if } |\phi| = 1 \\
                   X,         & \mbox{ if } \phi = \emptyset
                \end{array}
        \right.
	 && \left\{ \begin{array}{ll}
                   \emptyset, & \mbox{ if } |\psi| > 1 \\
                   \psi,         & \mbox{ if } |\psi| = 1 \\
                   X,         & \mbox{ if } \psi = \emptyset
                \end{array}
        \right.
	  \\ \hline

	\mbox{\bf bottom} & \term{{\cal X}_0}{\bot}{{\cal X}_1}
	 & \derivedir{\phi}{{\bot}} & \{ x_1 \memberof X_1 \mid (\forall x_0 \memberof X_0)\; x_0 \not\in \phi \}
	  & \deriveinv{\psi}{{\bot}} & \{ x_0 \memberof X_0 \mid (\forall x_1 \memberof X_1)\; x_1 \not\in \psi \} \\
	\multicolumn{2}{|c|}{\bot = \emptyset \subseteq \product{X_0}{X_1}}
	 && \left\{ \begin{array}{ll}
                   \bot_{\subpowerof{{\cal X}_1}} = \emptyset, & \mbox{ if } \phi \neq \emptyset = \bot_{\subpowerof{{\cal X}_0^{\rm op}}} \\
                   \top_{\subpowerof{{\cal X}_1}} = X_1,         & \mbox{ if } \phi =    \emptyset = \bot_{\subpowerof{{\cal X}_0^{\rm op}}}
                \end{array}
        \right.
	 && \left\{ \begin{array}{ll}
                   \bot_{\subpowerof{{\cal X}_0^{\rm op}}} = \emptyset, & \mbox{ if } \psi \neq \emptyset = \bot_{\subpowerof{{\cal X}_1}} \\
                   \top_{\subpowerof{{\cal X}_0^{\rm op}}} = X_0,         & \mbox{ if } \psi =    \emptyset = \bot_{\subpowerof{{\cal X}_1}}
                \end{array}
        \right. \\
	 && \multicolumn{2}{|c|}{\mbox{ the nonbottom test operator }}
	  & \multicolumn{2}{|c|}{\mbox{ the top test operator }}
	  \\ \hline

	\mbox{\bf complement} & \term{{\cal X}}{\not\geq}{{\cal X}}
	 & \derivedir{\phi}{\not\geq} & X \setminus \phi
	  & \deriveinv{\psi}{\not\geq} & X \setminus \psi
	  \\
	&&& \{ x' \memberof X \mid x \not\geq x' \mbox{ for all } x \memberof \phi \}
	  && \{ x \memberof X \mid x \not\geq x' \mbox{ for all } x' \memberof \psi \}
	  \\ \hline
 \end{array}$
\end{center}
\caption{{\bf Special relations} \label{relations}}
\end{table}

The following special derivation inequalities are used in the proof of the Equivalence theorem.
For any closed relation $\term{{\cal X}_0^\mu}{\mu}{{\cal X}_0^\mu}$
between induced orders,
\begin{flushleft}$\begin{array}{r@{\hspace{2mm}}l}
	{\bf direct} &
	\left\{\mbox{\begin{tabular}{c}
		``the direct derivation is within \\
		the upper bounds of the universal ideal quantification'' \\
		$\derivedir{\phi}{\mu}
			= \phi \tensorimplysource \mu
			\preceq (\mu \tensorimplysource \phi) \tensorimplysource (\mu \tensorimplysource \mu)
			= (\mu \tensorimplysource \phi) \tensorimplysource {\cal X}_1^{\mu}
			= \upper{(\mu \tensorimplysource \phi)}{{\cal X}_1^{\mu}}$ \\
		hence
		$\derivedir{\phi}{\mu}
			= \derivedir{\phi}{\mu} \wedge \upper{(\mu \tensorimplysource \phi)}{{\cal X}_1^\mu}$
	\end{tabular}}\right. \\ \\
	{\bf inverse} &
	\left\{\mbox{\begin{tabular}{c}
		``the inverse derivation is within \\
		the lower bounds of the universal filter quantification'' \\
		$\deriveinv{\psi}{\mu}
			= \mu \tensorimplytarget \psi
			\preceq (\mu \tensorimplytarget \mu) \tensorimplytarget (\psi \tensorimplytarget \mu)
			= {\cal X}_0^{\mu} \tensorimplytarget (\psi \tensorimplytarget \mu)
			= \looer{(\psi \tensorimplytarget \mu)}{{\cal X}_0^{\mu}}$ \\
		hence
		$\deriveinv{\psi}{\mu}
			= \looer{(\psi \tensorimplytarget \mu)}{{\cal X}_0^\mu} \wedge \deriveinv{\psi}{\mu}$
	\end{tabular}}\right. \\ \\
	{\bf upper} &
	\left\{\mbox{\begin{tabular}{c}
		``the upper bounds are within \\
		the direct derivation along negation of the universal ideal quantification'' \\
		$\upper{\phi}{{\cal X}_0^{\mu}}
			= \phi \tensorimplysource {\cal X}_0^{\mu}
			\preceq (\mu \tensorimplysource \phi) \tensorimplysource (\mu \tensorimplysource {\cal X}_0^{\mu})
			= (\mu \tensorimplysource \phi) \tensorimplysource {\neg}{\mu}
			= \derivedir{(\mu \tensorimplysource \phi)}{{\neg}{\mu}}$ \\
		hence
		$\upper{\phi}{{\cal X}_0^{\mu}}
			= \upper{\phi}{{\cal X}_0^\mu} \wedge \derivedir{(\mu \tensorimplysource \phi)}{{\neg}{\mu}}$
	\end{tabular}}\right. \\ \\
	{\bf lower} &
	\left\{\mbox{\begin{tabular}{c}
		``the lower bounds are within \\
		the inverse derivation along negation of the universal filter quantification'' \\
		$\looer{\psi}{{\cal X}_1^{\mu}}
			= {\cal X}_1^{\mu} \tensorimplytarget \psi
			\preceq ({\cal X}_1^{\mu} \tensorimplytarget \mu) \tensorimplytarget (\psi \tensorimplytarget \mu)
			= {\neg}{\mu} \tensorimplytarget (\psi \tensorimplytarget \mu)
			= \deriveinv{(\psi \tensorimplytarget \mu)}{{\neg}{\mu}}$ \\
		$\looer{\psi}{{\cal X}_1^{\mu}}
			= \deriveinv{(\psi \tensorimplytarget \mu)}{{\neg}{\mu}} \wedge \looer{\psi}{{\cal X}_1^\mu}$
	\end{tabular}}\right.
\end{array}$\end{flushleft}


\section{Wille's Concept Lattice Construction}

\begin{Proposition} Let ${\cal X}_0$ and ${\cal X}_1$ be any pair of orders.
	The concept lattice for the top relation
	$\term{{\cal X}_0}{\top}{{\cal X}_1}$
	is the unit:
	\begin{equation}
		{\bf CL}\triple{{\cal X}_0}{{\cal X}_1}{\top}
		= \{ \pair{X_0}{X_1} \}
		= {\bf 1}
		\label{top}
	\end{equation}
\end{Proposition}

\begin{Proposition} Let
	$\morphism{{\cal Y}_0}{f_0}{{\cal X}_0}$ and $\morphism{{\cal Y}_1}{f_1}{{\cal X}_1}$
	be any pair of {\em surjective\/} monotonic functions
	and $\term{{\cal X}_0}{\mu}{{\cal X}_1}$ be any closed relation.
	The concept lattice for the inverse image relation
	$\term{{\cal Y}_0}{f_0^\triangleright \circ \mu \circ f_1^\triangleleft}{{\cal Y}_1}$
	consists of inverse ideal image of extents:
	\begin{equation}
		{\bf CL}_{\rm ext}\triple{{\cal Y}_0}{{\cal Y}_1}{f_0^\triangleright \circ \mu \circ f_1^\triangleleft}
		= \{ f_0^\triangleright \circ \phi_0 \mid
		     \phi_0 \in {\bf CL}_{\rm ext}\triple{{\cal X}_0}{{\cal X}_1}{\mu} \}
		\label{invimage}
	\end{equation}
\end{Proposition}

\begin{Corollary} Let
	$\term{{\cal X}_0}{\mu}{{\cal X}_1}$ and $\term{{\cal Y}_0}{\nu}{{\cal Y}_1}$
	be any closed relations.
	The concept lattice for the product relation
	$\term{\product{{\cal X}_0}{{\cal Y}_0}}{\product{\mu}{\nu}}{\product{{\cal X}_1}{{\cal Y}_1}}$
	consists of products of extents:
	\begin{equation}
		{\bf CL}_{\rm ext}\triple{\product{{\cal X}_0}{{\cal Y}_0}}{\product{{\cal X}_1}{{\cal Y}_1}}{\product{\mu}{\nu}}
		= \{ \product{\phi_0}{\phi_1} \mid
		     \phi_0 \in {\bf CL}_{\rm ext}\triple{{\cal X}_0}{{\cal X}_1}{\mu},
		     \psi_0 \in {\bf CL}_{\rm ext}\triple{{\cal Y}_0}{{\cal Y}_1}{\nu} \}
		\label{product}
	\end{equation}
	Thus,
	the lattice of a product is the product of the lattices
	\[	{\bf CL}_{\rm ext}\triple{\product{{\cal X}_0}{{\cal Y}_0}}{\product{{\cal X}_1}{{\cal Y}_1}}{\product{\mu}{\nu}}
		\cong
		\product{{\bf CL}_{\rm ext}\triple{{\cal X}_0}{{\cal X}_1}{\mu}}{{\bf CL}_{\rm ext}\triple{{\cal Y}_0}{{\cal Y}_1}{\nu}} \]
\end{Corollary}

\begin{Proposition}
	Let $\term{{\cal X}}{\mu}{{\cal Y}}$ and $\term{{\cal X}}{\nu}{{\cal Z}}$
	be any closed relations,
	and let $\term{{\cal X}}{(\mu,\nu)}{{\cal Y}{+}{\cal Z}}$ be their pairing
	(called {\bf apposition} in \cite{Wille}).
	The concept lattice for the pairing consists of meets of extents:
	\begin{equation}
		{\bf CL}_{\rm ext}\triple{{\cal X}}{{\cal Y}{+}{\cal Z}}{(\mu,\nu)}
		= \{ \alpha \wedge \beta \mid
		     \alpha \in {\bf CL}_{\rm ext}\triple{{\cal X}}{{\cal Y}}{\mu},
		     \beta \in {\bf CL}_{\rm ext}\triple{{\cal X}}{{\cal Z}}{\nu} \}
		\label{pairing}
	\end{equation}
\end{Proposition}
\begin{proof}
	By the above facts,
	a concept of the pairing can be described as a triple
	$\pair{\phi}{(\psi,\zeta)}$
	where
	$\phi \memberof \powerof{{\cal X}^{\rm op}}$,
	$\psi \memberof \powerof{{\cal Y}}$ and
	$\zeta \memberof \powerof{{\cal Z}}$
	are subsets satisfying the conditions
	$\derivedir{\phi}{\mu} = \psi$,
	$\derivedir{\phi}{\nu} = \zeta$ and
	$\phi = \deriveinv{\psi}{\mu} \wedge \deriveinv{\zeta}{\nu}$.
	By defining
	$\alpha = \deriveinv{\psi}{\mu}$
	and
	$\beta = \deriveinv{\zeta}{\nu}$,
	we see that any pairing concept is of the form given by description~\ref{pairing},
	since
	$\pair{\alpha}{\phi} \in {\bf CL}\triple{{\cal X}}{{\cal Y}}{\mu}$
	 and
	$\pair{\beta}{\zeta} \in {\bf CL}\triple{{\cal X}}{{\cal Z}}{\nu}$.
	On the other hand,
	assume that
	$\pair{\alpha}{\phi} \in {\bf CL}\triple{{\cal X}}{{\cal Y}}{\mu}$
	 and
	$\pair{\beta}{\zeta} \in {\bf CL}\triple{{\cal X}}{{\cal Z}}{\nu}$
	are arbitrary concepts in the component lattices.
	Since
	$\closure{(\alpha \wedge \beta)}{\mu}
	 \preceq \closure{\alpha}{\mu}
	 = \alpha$
	and
	$\closure{(\alpha \wedge \beta)}{\nu}
	 \preceq \closure{\beta}{\nu}
	 = \beta$,
	we have
	$\alpha \wedge \beta
	= \closure{(\alpha \wedge \beta)}{\mu} \wedge \closure{(\alpha \wedge \beta)}{\nu}$.
	So the triple
	$\pair{\phi}{(\psi,\zeta)}$
	where
	$\phi \define \alpha \wedge \beta$,
	$\psi \define \derivedir{\phi}{\mu}$
	and
	$\zeta \define \derivedir{\phi}{\nu}$
	are subsets satisfying the pairing conditions.
\end{proof}

 \end{document}